\documentclass[11pt]{article}
\usepackage[preprint]{neurips_2026}
\usepackage[expansion=false]{microtype}
\AtBeginDocument{\let\cite\citep}
\usepackage[utf8]{inputenc}
\usepackage{booktabs}
\usepackage{amsmath}
\usepackage{enumitem}

\usepackage{amssymb}
\usepackage{amsthm}
\usepackage{tikz}
\usetikzlibrary{positioning,arrows.meta,calc,shapes,backgrounds,fit}
\usepackage{pgfplots}
\pgfplotsset{compat=1.17}
\usepackage{makecell}
\usepackage{colortbl}
\usepackage{adjustbox}
\usepackage{graphicx}
\usepackage{tabularx}
\usepackage{url}
\usepackage{natbib}
\providecommand{\Description}[1]{}   % no-op: acmart accessibility macro
\definecolor{atpblue}{HTML}{1F4E79}   % admission / authority
\definecolor{atpteal}{HTML}{2C7A7B}   % committed stores
\definecolor{atpgreen}{HTML}{2F855A}  % projection / ATP-safe
\definecolor{atpamber}{HTML}{C05621}  % unsafe baseline / violation
\definecolor{atpslate}{HTML}{4A5568}  % proposers / neutral
\tikzset{
  proposer/.style={draw=atpslate, fill=atpslate!8, rounded corners=2pt, align=center,
                   font=\small, inner sep=4pt, minimum height=8mm},
  gate/.style={draw=atpblue, fill=atpblue!12, rounded corners=2pt, align=center,
               font=\small, very thick, inner sep=5pt, minimum height=10mm},
  store/.style={draw=atpteal, fill=atpteal!10, rounded corners=2pt, align=center,
                font=\small, inner sep=4pt, minimum height=8mm},
  view/.style={draw=atpgreen, fill=atpgreen!10, rounded corners=2pt, align=center,
               font=\small, inner sep=4pt, minimum height=8mm},
  flow/.style={-{Latex[length=2.2mm]}, atpslate, semithick},
  commitflow/.style={-{Latex[length=2.2mm]}, atpteal, semithick},
  feedback/.style={-{Latex[length=2.2mm]}, atpgreen, semithick, dashed},
  orch/.style={-{Latex[length=2.2mm]}, atpslate!70, thin, dotted},
  planelabel/.style={font=\footnotesize\sffamily\bfseries},
}
\newcommand{\rowhead}{\rowcolor{atpblue!10}}
\newcommand{\Mnemo}{\textsf{Mnemosyne}}
\newcommand{\ATP}{\textsf{ATP}}
\newcommand{\StateView}{\textsf{StateView}}
\newcommand{\CTL}{\textsf{CTL}}
\newcommand{\ACR}{\textsf{ACR}}
\newcommand{\ACRs}{\ACR{}s}
\newcommand{\PNA}{\textsf{PNA}}
\newcommand{\DA}{\textsf{DA}}
\newcommand{\ESS}{\textsf{ESS}}
\newcommand{\DCC}{\textsf{DCC}}
\newcommand{\OC}{\textsf{OC}}

\newcommand{\AET}{\textsf{AET}}

\newcommand{\IDC}{\textsf{IDC}}
\newcommand{\EPR}{\textsf{EPR}}
\newcommand{\SEA}{\textsf{SEA}}

\newtheorem{invariant}{Invariant}
\newtheorem{theorem}{Theorem}
\newtheorem{corollary}{Corollary}

\newcounter{atpalgorithm}
\usepackage{etoolbox}
\AtBeginEnvironment{tabular}{\small}
\providecommand{\Description}[2][]{}
\newcommand{\keywordsblock}[1]{\par\medskip\noindent\textbf{Keywords:} #1\par}
\newtheorem{proposition}[theorem]{Proposition}
\newtheorem{lemma}[theorem]{Lemma}
\usepackage{titletoc}% partial ToC for the supplement
\title{Mnemosyne: Agentic Transaction Processing for Validating and Repairing AI-generated Workflows}
\author{Edward Y.\ Chang\\Stanford University \and Longling Geng\\Stanford University}
\date{}
\begin{document}
\maketitle
%\begingroup\small\noindent\textbf{Preprint.}\ This is the authors' full version,
%comprising the main paper under submission to PVLDB Volume 20 together with its
%complete supplementary material (Appendices~S1--S22).\par\endgroup
%\medskip
\begingroup\small\noindent\textbf{Artifact Availability:}
The source code, data, and/or other artifacts are available at
\url{https://github.com/eyuchang/Mnemosyne/tree/arxiv-atp-rq1-rq9b-r8-v2}.\par\endgroup
\medskip
\keywordsblock{Agentic transaction processing, LLM agents, Transaction processing, Workflow validation and repair, Data integrity}
\begin{abstract}
LLMs increasingly generate workflow actions and repairs that may be well formed yet stale,
infeasible, conflicting, or destructive of their own evidence. We introduce \emph{Agentic
Transaction Processing} (\ATP{}), which treats generated actions as untrusted proposals until
deterministic admission accepts them under an executable constraint set $\mathcal{C}$. Its
two-sided principle is: \emph{a proposal is not truth, and no proposal foresees every disruption}.
Anything may propose, but only the runtime admits and commits; unforeseen disruptions trigger
bounded reactive repair whose output re-enters admission. \Mnemo{} realizes ATP with an append-only
transition log, effective-state projection, dependency-safe compensation, and active contract
records. Under stated assumptions and relative to $\mathcal{C}$, we prove four safety
properties---authority separation, serial-equivalent generative admission, evidence-preserving
repair, and obligation containment---and establish bounded reactive repair. Across nine safety
benchmarks and a four-case Temporal SDK comparison, ATP rejects every targeted violation while
admitting valid work. A matched data-size sweep measures $5.2$--$6.6\%$ incremental throughput
cost over the same local durable commit path and exposes local saturation. In a companion
scheduling harness, local repair edits nearly an order of magnitude fewer operations than global
recompute; a 12-scenario interruption stress test rejects every stale recovery candidate without
losing observations or producing invalid commits. Two bounded pilots route 80 proposals from four
heterogeneous LLMs through the same gate with zero invalid commits; 24 of 40 mid-execution
proposals are admitted and 16 are rejected, including four explicit safety rejections. 
\end{abstract}

\section{Introduction}

Transactions protect submitted work through isolation, logging, recovery, and integrity
constraints~\cite{gray1993transaction,bernstein2009principles,weikum2001transactional,mohan1992aries};
workflows add retries, timeouts, idempotency, and compensation
\cite{garciamolina1987sagas,ludascher2006scientific,russell2005workflow}. Both assume an
application-authored request meaningful enough to transact. Agentic workflows break that
assumption: an LLM, solver, learned guard, or fired contract may generate a well-formed but stale
or infeasible candidate~\cite{sagallm2025vldb,yao2023react,yao2023treeofthoughts,shinn2023reflexion,wang2023voyager,park2023generativeagents,wang2023surveyllmagents,chang2024pathagi1}.
ATP is proposer-neutral; \textsc{ProposerQualityBench} therefore includes solver-like,
rule-based, random, adversarial, and LLM-like proposers.

Consider a case from multi-agent planning~\cite{sagallm2025vldb}. In a family's
Thanksgiving plan, the father lands in Boston and picks up Grandma on the way home so dinner
starts on time. His flight is delayed, so an LLM agent replans the afternoon's pickups and
drives. The new schedule parses as valid, yet two faults slip through: it reuses the off-peak
drive time from the original 3\,pm plan, so the rescheduled pickup now lands in rush-hour
traffic and is physically impossible; and after reshuffling the drivers, it silently drops
Grandma's pickup, leaving her stranded so dinner never starts.

A conventional transaction layer would commit this well-formed schedule and inherit both faults.
\ATP{} instead routes each revised assignment through a deterministic gate against the declared
constraint set $\mathcal{C}$ over the current effective state (the \StateView{}). The infeasible
pickup is rejected because its travel time violates $\mathcal{C}$; the replan is rejected for
dropping Grandma's pickup, since a committed obligation cannot silently disappear. Feasible
reassignments commit; the rest are rejected with a recorded reason. The guarantee comes from
outside the model: a stronger planner gets more of the day right, but none can erase a standing
commitment or commit a physically impossible plan.

Admission stops bad proposals, but it cannot prevent an \emph{unforeseen} disruption. If a highway
closure invalidates the committed airport leg, the affected contract's watcher proposes readiness
from the admitted closure report. Only after admission commits that lifecycle state may its
continuation construct a bounded local repair through LCRP, re-timing the disrupted leg and its
dependents while leaving the rest of the day untouched. The repair then re-enters ordinary
admission. This is \emph{governed exception handling}: a conventional \texttt{catch} executes,
whereas an \ACR{} continuation only proposes. Static admission and reactive repair therefore share
one authority boundary; \S\ref{sec:lcrp} specifies it, and RQ1, RQ3, and RQ6 evaluate both paths,
including under live LLM proposers.

The relevant boundary is therefore the transition from generated intent to durable truth. The first
question is not only ``can this transaction commit?'' but \emph{should this proposal be admitted as
a transaction at all?} \emph{Agentic Transaction Processing} (\ATP{}) governs that decision. A
workflow remains a sequence of actions; a transaction remains the unit of work that commits under
ACID. ATP decides when a generated workflow step earns that status. LLMs can draft and repair at
great speed, but may hallucinate or endorse their own plans, so the proposer cannot be its own
judge. ATP retains the automation while treating every output as untrusted. Its central rule is:
\begin{center}
\emph{A proposal is not truth, and no proposal foresees every disruption.}
\end{center}
An LLM may propose, an active obligation may wake, and a driver may orchestrate; only the
transaction layer admits and commits.

This creates a separation classical transaction processing did not need: relative to
$\mathcal{C}$, model competence changes proposal usefulness, not the ability to append invalid
truth. The data layer---a committed-transition log, effective-state projection, and deterministic
gate---governs correctness. ATP therefore complements agent-side discipline such as
SagaLLM~\cite{sagallm2025vldb}: one disciplines the proposer; the other governs the committer.

%\vspace{-.06in}
\paragraph{Contributions.}
\begin{enumerate}[leftmargin=1.2em,itemsep=2pt,topsep=1pt]
\item \textbf{The ATP model: a new database-systems problem and its transaction model.} A
generated action must \emph{earn} transaction authority before becoming committed truth
(\S\ref{sec:motivation}). ATP formalizes that boundary: every generated output is a proposal
that only deterministic admission under a declared constraint set $\mathcal{C}$ may commit,
with \emph{active contract records} (\ACRs{}) as log-resident obligations whose readiness and remedy transitions are
themselves proposals (\S\ref{sec:model}--\S\ref{sec:invariants}).
\item \textbf{Guarantees for the boundary, static and reactive.} Four safety theorems (authority
separation, serial-equivalent generative admission, evidence-preserving repair, and obligation
containment), plus a localized repair protocol (LCRP) whose output re-enters
admission, a bounded-recovery contract, and a reduction of recursive recovery to sequential
recovery (\S\ref{sec:theorems}).

\item \textbf{The \Mnemo{} runtime and its evidence.} An executable ATP substrate over a
committed-transition log, effective-state projection, deterministic admission, dependency-safe
compensation, and PostgreSQL-backed concurrency (\S\ref{sec:arch-spec}); evaluated through nine
safety benchmarks, an executable Temporal SDK comparison, worker- and data-scale cost audits,
a bounded-repair containment study, a 12-scenario recursive-recovery stress test, and live-proposer
pilots in which four heterogeneous LLMs submit static plans and mid-execution repairs through the
same boundary under an external cross-episode scorer (\S\ref{sec:eval}).

\end{enumerate}

\section{Agentic Transaction Processing}
\label{sec:atp}

\begin{figure}[t]
\centering
\adjustbox{max width=\linewidth}{%
\begin{tikzpicture}[font=\scriptsize, node distance=1.5mm,
  inv/.style={draw=atpblue, fill=atpblue!8, rounded corners=1.5pt, align=left,
              inner sep=2.5pt, minimum height=5mm, text width=27mm},
  prop/.style={draw=atpgreen!70!black, fill=atpgreen!10, rounded corners=1.5pt, align=left,
              inner sep=2.5pt, minimum height=5mm, text width=30mm},
  lk/.style={-{Latex[length=2mm]}, atpslate, semithick},
  lkd/.style={-{Latex[length=2mm]}, atpslate, semithick, dashed},
  lab/.style={font=\tiny, fill=white, inner sep=1pt}]
% headers
\node[font=\scriptsize\bfseries, atpblue]  (hi) at (0,0.75) {Invariants: per-event, executable, counted};
\node[font=\scriptsize\bfseries, atpgreen!60!black] (hp) at (7.6,0.75) {Properties: all-execution contracts};
% invariant column
\node[inv] (i1) at (0,0)     {no invalid commit};
\node[inv] (i5) at (0,-0.62) {stale-world rejection};
\node[inv] (i4) at (0,-1.24) {conflict rejection};
\node[inv] (i7) at (0,-1.86) {evidence preservation};
\node[inv] (i6) at (0,-2.48) {non-authoritative wakeup};
\node[inv] (i8) at (0,-3.10) {bounded, re-admitted repair};
\node[inv] (i2) at (0,-3.72) {effective-state separation};
\node[inv] (i3) at (0,-4.34) {no orphaned dependents};
% property column
\node[prop] (p1) at (7.6,-0.31) {Authority separation\\(\PNA{}$+$\DA{}$\Rightarrow$\IDC{})};
\node[prop] (p2) at (7.6,-1.24) {Serial-equivalent admission (\SEA{})};
\node[prop] (p3) at (7.6,-1.86) {Evidence-preserving repair (\EPR{})};
\node[prop] (p4) at (7.6,-2.48) {Obligation containment (\OC{})};
\node[prop] (p5) at (7.6,-3.10) {Bounded reactive repair (LCRP)};
\node[prop] (p6) at (7.6,-3.72) {Effective-state soundness (\ESS{})};
\node[prop] (p7) at (7.6,-4.34) {Dependency-closed compensation (\DCC{})};
% links
\draw[lk] (i1.east) -- node[lab]{Thm.~\ref{thm:idc}} (p1.west|-i1.east);
\draw[lk] (i5.east) -- (p1.west|-i5.east);
\draw[lk] (i4.east) -- node[lab]{Thm.~\ref{thm:iso}} (p2.west);
\draw[lk] (i7.east) -- node[lab]{Thm.~\ref{thm:epr}} (p3.west);
\draw[lk] (i6.east) -- node[lab]{Thm.~\ref{thm:oc}} (p4.west);
\draw[lk] (i8.east) -- node[lab]{Props.~\ref{prop:brr}--\ref{prop:rrr}} (p5.west);
\draw[lkd] (i2.east) -- node[lab]{by construction} (p6.west);
\draw[lkd] (i3.east) -- node[lab]{by construction} (p7.west);
% audit band
\node[draw=atpamber!80!black, fill=atpamber!10, rounded corners=1.5pt, align=center,
      inner sep=2.5pt, text width=98mm, font=\tiny] (aud) at (3.8,-5.15)
  {\AET{} (audit and traceability) retains the evidence every link reads. RQ1's falsification
   benchmarks attack the left column; violation counts (Fig.~\ref{fig:violation-suppression})
   test the gate-closure assumption the proofs stand on.};
\end{tikzpicture}}
\caption{The assurance web: each executable invariant (left, checked by the gate on every event)
is carried by a theorem, under the trusted constraint set $\mathcal{C}$ and gate closure, to the
all-execution property it guarantees (right). Dashed links hold by construction; the audit band
supplies the evidence substrate and names the evaluation's role.}
\Description{A two-column diagram links eight per-event invariants on the left to seven
system-level properties on the right; solid arrows are labeled with theorem numbers, dashed
arrows hold by construction, and a band beneath notes that audit and traceability retain
evidence and that the RQ1 falsification benchmarks attack the invariant column.}
\label{fig:assurance-web}
\end{figure}
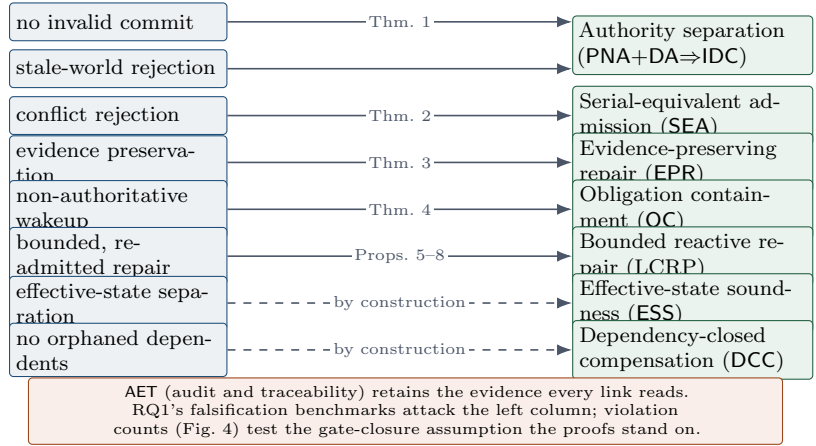

\label{sec:motivation}
Generated workflow actions fail in ways ACID does not address: an invalid transition, a
compensation that orphans effective dependents, conflicting or stale-world proposals, or a repair
that destroys the evidence that triggered it. ACID governs a transaction \emph{after} the system
accepts it~\cite{gray1993transaction,bernstein2009principles,weikum2001transactional}; \ATP{}
governs whether a generated proposal should be admitted as a transaction at all. This section
builds that guarantee in three linked levels, and Figure~\ref{fig:assurance-web} is its map.
\emph{Properties} (\S\ref{sec:properties}) are the contracts: declarative guarantees quantified
over all executions, stating what a generative proposer must never be able to commit.
\emph{Invariants} (\S\ref{sec:invariants}) are their executable form: per-event predicates the
admission gate checks on every proposal, which code can enforce, tests can falsify, and the
evaluation counts. \emph{Theorems} (\S\ref{sec:theorems}) are the bridge: proofs that a runtime
maintaining the invariants on every event, under the trusted constraint set $\mathcal{C}$ and
gate closure, satisfies the contracts on every execution. The transaction model itself
(\S\ref{sec:model}) supplies the objects all three levels quantify over.

\subsection{The Agentic Transaction Model}
\label{sec:model}

The model has one moving part and one gate, and Table~\ref{tab:notation} fixes the
vocabulary in one place so no term must be hunted down later. An LLM proposes; the runtime
admits, commits, and projects. A \emph{proposal} $p$ is an uncommitted candidate action
generated by an LLM (the proposer this paper focuses on; the model is proposer-agnostic, and
\S\ref{sec:eval} exercises other proposer classes only to show that the guarantees do not
depend on who proposes). \emph{Admission} checks $p$ deterministically against the declared
constraint set $\mathcal{C}$ and the current effective state; an admitted proposal is durably
appended to the committed-transition log (\CTL{}), the append-only source of committed truth,
and projected into \StateView{}, the effective state that future admissions read. A rejected
proposal leaves only a rejection event. External effects stage through an outbox
after commit, so nothing reaches the world that the log did not record.

\begin{table}[t]
\centering
\caption{Notation. All guarantees are stated relative to $\mathcal{C}$.}
\label{tab:notation}
\footnotesize
\setlength{\tabcolsep}{4pt}
\begin{tabular}{p{0.12\linewidth}p{0.83\linewidth}}
\toprule
\rowhead
Symbol & Meaning \\
\midrule
$p$ & Proposal: uncommitted candidate action from LLM; carries target entity, operation, dependencies, and assumptions. \\
$\mathcal{C}$ & Declared, executable constraint set the validator enforces: finite-state, dependency, compensation, stale-world, conflict-scope, and evidence-preservation rules. \\
\CTL{} & Committed-transition log: append-only source of committed truth (admitted transitions, \ACR{} lifecycle and recovery events). \\
\StateView{} & Effective state: replay of committed transitions excluding compensated or superseded records; the only state admission reads. \\
scope & Conflict scope: tenant, entity, or recovery region over which two proposals may not both become effective without revalidation; overlaps serialize or reject. \\
\ACR{} & Active contract record: a durable obligation in \CTL{} with trigger, guard, continuation, compensation, and expiry. A watcher proposes readiness; admission records it; only then may the continuation propose a remedy. No stage writes domain truth directly. \\
\bottomrule
\end{tabular}
%\vspace{-.1in}
\end{table}

Algorithm~\ref{alg:admission} gives the executable contract used by the implementation and by the RQ1--RQ6 experiments. The details of $\mathcal{C}$ are application-specific, but the authority pattern is fixed: proposal generation can vary, while admission and commit remain deterministic runtime operations.

%\vspace{-.06in}
\subsection{Properties Beyond ACID}
\label{sec:properties}

ACID is sufficient for traditional transaction processing, where the submitted unit of
work is authored by a trusted application: atomicity, consistency, isolation, and durability
protect a transaction \emph{after} the system accepts
it~\cite{gray1993transaction,bernstein2009principles,weikum2001transactional,mohan1992aries}.
\ATP{} must in addition safeguard LLM-generated workflows \emph{before} acceptance. Eight
properties do that work, in two tiers: five gate-critical properties that
\S\ref{sec:theorems} proves as theorems (\S\ref{sec:tier1}), and three structural properties
the runtime enforces by construction and RQ1 exercises (\S\ref{sec:tier2}). For each we state
what it requires, why it must be preserved, that is, what commits if it is dropped, and which
classical guarantee it extends. These are not hygiene rules; each names a corruption channel
that a generative proposer opens and a trusted application never did.

\subsubsection{Tier 1: The Gate-Critical Properties}
\label{sec:tier1}

\paragraph{Authority separation (\PNA{} $+$ \DA{} $\Rightarrow$ \IDC{}; Theorem~\ref{thm:idc}).}
Two requirements compose into the paper's central guarantee. \emph{Proposal non-authority}
(\PNA{}): an LLM, a fired \ACR{}, or any other generator may propose, but none may directly
create committed truth. \emph{Deterministic admission} (\DA{}): every proposal passes a
deterministic validator enforcing $\mathcal{C}$ against effective state, and the verdict
depends on the proposal's content, never on who proposed it. Together they yield
\emph{intelligence-decoupled correctness} (\IDC{}): committed-state correctness relative to
$\mathcal{C}$ is independent of the competence, honesty, or learning behavior of the proposing
layer. Why it must be preserved: without it, the rush-hour pickup of \S1 commits, and every
hallucination, staleness, and sycophancy failure of the model becomes durable truth; committed
correctness would be exactly as good as the model on its worst day. This is atomicity and
consistency re-founded for an untrusted author.

\paragraph{Serial-equivalent generative admission (\SEA{}; Theorem~\ref{thm:iso}).}
Concurrent proposers are admitted as if in some serial order over their declared conflict
scopes, so committed state never reflects two conflicting proposals or a duplicate of one. Why:
agentic proposers retry, race, and resubmit; without \SEA{}, two plausible plans for the same
vehicle both commit, or a retried payment commits twice. Classical isolation orders
transactions from trusted applications; \SEA{} extends that ordering to generators that cannot
be assumed to coordinate.

\paragraph{Evidence-preserving repair (\EPR{}; Theorem~\ref{thm:epr}).}
An admitted repair may not discharge its own trigger by deleting, compensating, or obscuring
the evidence that justified it, unless admission verifies the triggering condition is actually
resolved under $\mathcal{C}$. Why: the cheapest way for a generated repair to ``succeed'' is to
silence the alarm, delete the failing tests, suppress the alert, compensate the diagnostic
record, and every downstream check then reads green while the fault persists. \EPR{} has no
classical analogue, because classical validation never had to distrust the author of the
repair; it is the property that separates \ATP{} from audit logging.

\paragraph{Obligation containment (\OC{}; Theorem~\ref{thm:oc}).}
An \ACR{} watcher may propose readiness when its trigger fires; only an admitted ready or
breached lifecycle state may resume a continuation that \emph{proposes}, and neither stage can
mutate \CTL{} or \StateView{} directly. Why: active memory is the mechanism
that lets the system react to the unforeseen, and it is also a standing army of triggers; if a
stale, spurious, or adversarially planted wakeup could write, memory would be an ungoverned
actor inside the transaction layer. \OC{} extends the event-condition-action trigger of active
databases, with the action demoted from a write to a proposal.

\paragraph{Bounded reactive repair (LCRP; Proposition~\ref{prop:brr}).}
When an unforeseen disruption invalidates committed work, the repair edits only the affected
region, under a declared edit radius and iteration cap, and the repair itself re-enters
admission. Why: without the bound, a single mid-execution disruption licenses an unbounded,
unguarded re-plan, precisely the uncontrolled rewrite the gate exists to prevent; without
re-admission, recovery becomes a side door around \DA{}. This extends saga-style recovery from
pre-registered compensations to repairs that are \emph{constructed} at run time yet still earn
commit authority like any other proposal.

\begin{figure}[t]
\centering
\small
\refstepcounter{atpalgorithm}\label{alg:admission}
\begin{tabular}{@{}p{0.045\linewidth}p{0.86\linewidth}@{}}
\toprule
\multicolumn{2}{@{}l@{}}{\textbf{Algorithm~\theatpalgorithm: ATP admission under constraint set $\mathcal{C}$.}}\\
\midrule
1 & Parse $p$; check tenant, idempotency and operation keys, and the declared conflict scope. \\
2 & Read effective \StateView{} and retained evidence, never speculative or rejected history. \\
3 & Apply $\mathcal{C}$ (finite-state, stale-world, dependency closure, compensation safety); if $p$ is a repair, require its triggering evidence to stay effective and queryable. \\
4 & Serialize or reject overlapping conflict scopes; \ACR{} readiness and remedy outputs may only \emph{propose}, never mutate \CTL{} or \StateView{}. \\
5 & Atomically append and project to \StateView{}, else record a queryable rejection reason. \\
\bottomrule
\end{tabular}
\Description{Compact ATP admission contract: parse and scope-check; read effective state and evidence; apply the constraint set including evidence-preserving repair; serialize conflicting scopes while active contracts may only propose; then atomically commit or record a rejection.}
\end{figure}

\subsubsection{Tier 2: The Structural Properties}
\label{sec:tier2}

Three further properties are not proved as separate theorems; the runtime enforces them by
construction, and RQ1's projection and compensation benchmarks exercise them.
\emph{Effective-state soundness} (\ESS{}): \StateView{} is derived only from effective
committed records, so admission never validates against compensated or superseded history; drop
it and the gate approves proposals against a world that no longer exists.
\emph{Dependency-closed compensation} (\DCC{}): a compensation is inadmissible if it would
orphan effective dependents, foreclosing the half-unwound bundle in which a refund is issued
while shipped siblings still depend on the reversed reservation; it extends saga compensation
with an effective-dependency check. \emph{Audit and traceability} (\AET{}): rejections,
wakeups, repairs, and compensations remain queryable as historical evidence; \AET{} is not
itself the primary guarantee, but \EPR{} and post-hoc accountability are only enforceable
because the evidence they read survives, which is durability extended from committed data to
the decision record around it. \AET{} is exercised on every artifact run: the falsification
benchmarks assert against verdicts and rejection reasons read back from the durable record, not
against in-memory return values.

\paragraph{Why this is not input validation.}
With the eight properties in view, the objection can be answered precisely. Classical input
validation and integrity constraints check \emph{data quality} on commands from a trusted,
authorized submitter. A generative proposer breaks that trust assumption. \ATP{} governs the prior
question, whether an untrusted \emph{generated} proposal may become a transaction request at all,
so it enforces an \emph{authority contract} over committed state rather than a richer predicate on
a trusted command: validation against the effective-state witness (\ESS{}), repairs that may not
destroy the evidence that justified them (\EPR{}), contracts that may wake but never write
(\OC{}), and concurrent proposals admitted as if serial (\SEA{}). The hardest of these,
evidence-preserving repair, has no classical analogue, because classical validation never had to
distrust the author of the command.

\subsection{Safety Invariants}
\label{sec:invariants}

A contract quantified over all executions cannot be checked by any single runtime event.
This section therefore restates each property of \S\ref{sec:properties} as an \emph{invariant}:
a decidable, per-event predicate the admission gate evaluates before every commit, which tests
can falsify and whose violations are the counts the evaluation reports.
Figure~\ref{fig:assurance-web} draws each invariant-to-property link and names the theorem that
carries it.

Throughout, fix a \emph{trusted constraint set} $\mathcal{C}$ consisting of the deterministic
admission validator, application finite-state constraints, dependency and compensation rules,
stale-world checks, evidence-preservation rules, and conflict-scope rules. A committed
transition is \emph{$\mathcal{C}$-valid} if the admission gate accepts it under this constraint
set and current effective state. All guarantees are relative to $\mathcal{C}$; if $\mathcal{C}$
omits a hazard, ATP does not claim to detect it. The results are relative systems-safety claims
rather than complexity-theoretic claims.

The seven executable invariants, each tagged with the property it carries, are: \emph{no
invalid commit} (non-empty violation set $\Rightarrow$ never written to \CTL{};
$\rightarrow$\,\DA{}); \emph{stale-world rejection} ($\rightarrow$\,\DA{});
\emph{proposal-conflict rejection} over declared scopes ($\rightarrow$\,\SEA{});
\emph{evidence preservation} ($\rightarrow$\,\EPR{}); \emph{non-authoritative wakeup}
($\rightarrow$\,\OC{}); \emph{effective-state separation} (full history retained, effective
history excludes compensated and superseded records; $\rightarrow$\,\ESS{}); and \emph{no
orphaned effective dependents} ($\rightarrow$\,\DCC{}). Their formal statements, with
additional invariants, are in the supplementary material \cite{Mnemosyne-Supp}.

\section{Guarantees: Admission Safety and Reactive Repair}
\label{sec:theorems}

We prove the Tier-1 properties of \S\ref{sec:tier1} as safety theorems: first for the admission of generated proposals, then for reactive repair under unforeseen disruptions (\S\ref{sec:lcrp}). Every theorem rests on the two assumptions stated next, so they precede the statements rather than trailing them. (Full proofs are in the supplementary material \cite{Mnemosyne-Supp}.)

\paragraph{Bounded executable constraint set.}
ATP does not assume an oracle for semantic correctness. In the implementation and
experiments, $\mathcal{C}$ is a finite executable constraint set over proposal
structure, tenant/entity scope, operation and idempotency keys, current
\StateView{}, dependency edges, retained evidence handles, admitted digital observations,
and application validators. Thus $\mathcal{C}$ is decidable for the artifact
workloads. In richer deployments, constructing a faithful $\mathcal{C}$ is part
of the application engineering problem rather than something ATP solves by
itself. ATP guarantees internal consistency relative to those admitted observations;
establishing their correspondence to the physical world is outside this paper's guarantee.

\paragraph{Gate closure.}
The theorems assume gate closure: proposer, solver, runtime-driver, benchmark,
and \ACR{}-wakeup code can emit proposal packages but cannot append domain truth
to \CTL{} except through admission. Mnemosyne's reference implementation is
organized by this boundary and includes regression tests for non-authoritative
wakeups, rejected proposals, and commit-boundary behavior, but this paper does
not claim a machine-checked noninterference proof of the Python codebase.
The formal claims therefore apply to executions that satisfy this gate-closure
assumption.

\paragraph{Where the difficulty lives.}
Theorem~\ref{thm:idc} is deliberately simple once the gate is trusted: if only
$\mathcal{C}$-valid transitions are appended, committed state is $\mathcal{C}$-valid.
The systems contribution is not that implication but the \emph{structure} of
$\mathcal{C}$ and the \emph{witness} the gate reads. A classical validator checks a
submitted command against static predicates; ATP's gate reads the effective-state projection
rather than raw history, enforces dependency-closed compensation, rejects stale-world and
conflicting proposals, and, hardest of all, enforces evidence-preserving repair
(Theorem~\ref{thm:epr}). Authority separation
(Theorem~\ref{thm:idc}) is the umbrella under which it and the other gate-critical properties
are instantiated.

\subsection{Admission Safety}

\begin{theorem}[Authority-Separation Theorem]
\label{thm:idc}
Let $\mathcal{C}$ be the trusted admission constraint set. Let the intelligence layer consist of arbitrary LLMs, solvers, agents, active-contract wakeups, learned repair policies, and benchmark adapters. Assume: (i) the intelligence layer can only emit proposals; (ii) the only operation that extends \CTL{} is append of an admitted transition; and (iii) the admission gate appends only proposals accepted under $\mathcal{C}$ and current effective state. Then committed-state correctness with respect to $\mathcal{C}$ is independent of the correctness, competence, honesty, or learning behavior of the intelligence layer.
\end{theorem}

This is the central systems claim, the formal counterpart of \S\ref{sec:tier1}'s first property:
the proposer governs usefulness only; nothing corrupts committed truth unless $\mathcal{C}$ itself
admits the transition.

Corollaries follow the same pattern for solver certificates (a certificate is evidence, not
admissibility) and for the Tier-2 properties \ESS{} and \DCC{}; their statements are in the
supplement.

\begin{theorem}[Serial-Equivalent Generative Admission]
\label{thm:iso}
Assume the declared conflict scopes are correct and complete for every incompatibility relevant to $\mathcal{C}$. By the conflict-scope contract (Table~\ref{tab:notation}), proposals with overlapping incompatible scopes are serialized or rejected at admission. Assume in addition that append to \CTL{} is atomic and that every dependent recovery proposal reads only committed effective state after its dependencies have committed. Then concurrent generative proposers are observationally equivalent to some serial admission order over committed transitions. Hidden conflicts outside the declared scopes are outside this guarantee. Moreover, cascading recovery is memoryless with respect to speculative proposal history: later recovery depends only on committed history and effective state, not on rejected, in-flight, or superseded proposal attempts.
\end{theorem}

%\vspace{-.06in}
\begin{corollary}[Storage-Level Idempotency Boundary]
In the PostgreSQL substrate, tenant-scoped uniqueness over event identifier, idempotency key, and recovery sequence position implements the concrete idempotency and conflict boundary used by Theorem~\ref{thm:iso}: concurrent duplicate proposals either converge to one canonical committed event or reject conflicting losers without producing duplicate committed truth.
\end{corollary}

%\vspace{-.06in}
\begin{theorem}[Evidence-Preserving Repair Safety]
\label{thm:epr}
Let a repair proposal be triggered by evidence $E$ of a failure condition $F$. If admission requires that either $F$ is resolved under $\mathcal{C}$, or $E$ remains effective and queryable after the repair, then no admitted repair can discharge its own trigger merely by deleting, compensating, or obscuring the evidence that justified it.
\end{theorem}

This theorem is stronger than audit traceability. Traceability says the alarm can be found later; evidence-preserving repair says the alarm cannot be silenced by the repair.

\paragraph{Decidability caveat for repair.}
Theorem~\ref{thm:epr} does not decide whether an arbitrary real-world failure has
truly been fixed. That question can be open-ended or outside the information
available to the runtime. ATP requires either an executable validator in
$\mathcal{C}$ that accepts the repair as resolving the triggering condition, or
continued retention of the triggering evidence. If neither condition can be
established, the safe behavior is rejection or escalation rather than admission.

%\vspace{-.06in}
\begin{theorem}[Obligation Containment]
\label{thm:oc}
If every \ACR{} is created by an admitted transition, every readiness or breach lifecycle change is itself admitted, and only a contract in a committed ready or breached state may resume a continuation that emits a proposal package, then no active contract can mutate domain truth except through the same admission boundary as any other proposal.
\end{theorem}

%\vspace{-.06in}
\begin{corollary}[Learning Non-Authority]
\label{cor:learn}
If learned artifacts can affect only proposal generation and ranking, and cannot weaken $\mathcal{C}$ except through an explicitly admitted governance transition, then learning cannot directly compromise committed-state correctness. By Theorem~\ref{thm:idc}, any domain effect of learning must pass the admission gate, so its worst outcome is over-rejection, delayed repair, or selection among admissible but suboptimal proposals, never an invalid commit. Whether learning improves remedies, reduces repeated failures, or improves later outcome and error analysis is outside the scope of this paper and is deferred to future work (see the scope statements in \S\ref{sec:conclusion}).
\end{corollary}

\paragraph{What these guarantees do and do not claim.}
All ATP guarantees are relative to the declared trusted constraint set $\mathcal{C}$. If $\mathcal{C}$ is incomplete, ATP cannot reject violations outside $\mathcal{C}$. The theorems do not prove convergence, optimality, monotonic improvement, unlimited liveness, or full-code bypass freedom. They prove a narrower systems property: intelligence, learning, and recovery memory remain non-authoritative. They may propose, rank, delay, or over-reject, but they cannot commit domain truth except through admission under $\mathcal{C}$.

\subsection{Reactive Repair for the Unforeseen}
\label{sec:lcrp}
Not every threat to committed state is a proposal: the world also produces failures no
proposal anticipated, a delayed flight, a machine breakdown, a road closure mid-execution, and a
perfectly admitted plan can be overtaken by an event it never modeled. This subsection enforces
the principle's second clause with \emph{reactive local repair}.

\ATP{} routes four exception classes through the same boundary: a \emph{pre-commit violation} is
rejected; an admitted \emph{post-commit disruption} may generate readiness proposals for affected \ACRs{}; a failed
\emph{external effect} returns through the inbox (\S\ref{sec:arch-spec}) as an observation; and a
\emph{temporal exception}, such as a timeout, likewise proposes an \ACR{} lifecycle change.
Only after admission commits readiness or breach may its continuation construct a remedy proposal.
RQ1, RQ3, and RQ6 evaluate the first two routes; effect and timeout routing remain outside this
paper's evaluation. Orthogonally, anticipated failures may use registered compensation
(Theorems~\ref{thm:epr}--\ref{thm:oc}), whereas unforeseen failures require a constructed repair.
Here an admitted ready or breached \ACR{} hands the disrupted region to a
\emph{localized cascading repair protocol} (LCRP) rather than recomputing the whole plan. The
Thanksgiving plan of \S1 supplies the running instance:
the afternoon has been admitted and is executing when the highway closure makes the committed
airport leg infeasible. No registered compensation covers a closed highway. The airport-leg
contract is exactly an \ACR{}: its trigger (the leg's feasibility assumption) produces a
readiness proposal; after admission commits the breach, its continuation hands LCRP the disrupted region, the airport leg and the pickups that causally depend
on it, while dinner preparation, the unaffected drives, and every already-completed step stay
committed and untouched.

\paragraph{The reactive loop.}
ATP requires only a bounded repair handler that returns a proposal with triggering evidence, scope,
and budget; LCRP is the scheduling-oriented realization evaluated here. The runtime runs one
loop per event (propose, admit, commit, monitor, repair), and on disruption LCRP
applies \emph{bounded edits} instead of a global recompute. It edits only the disrupted region under
an iteration cap $K$ and a displacement budget $\delta_{\max}$, where the \emph{edit radius} counts the
transitions (operations or jobs) touched. Each iteration revalidates feasibility against $\mathcal{C}$
and propagates effects only along the \emph{effective dependencies} of the disrupted region, so a
change ripples to a successor only if that successor actually depends on the changed record in the
current \StateView{}. The runtime, not the proposal, computes the reversibility class of each
candidate edit, so a repair cannot self-certify as safe; one that would cross an irreversible boundary
is vetoed. Figure~\ref{fig:lcrp-containment} draws the contrast with a global recompute. Crucially, the repair LCRP produces is itself a \emph{proposal} that re-enters the admission
gate, so a wrong reactive repair is caught exactly like any other proposal and can never become
committed truth.

\paragraph{Repair phases.}
A single repair round proceeds in three phases over the disrupted region. After inbox
deduplication, \emph{(I) Status update} submits the disruption report as a candidate observation
transition; only admission may commit the resulting invalidation, and the report remains queryable
as repair evidence. Repair therefore reads the admitted post-disruption \StateView{}, not the stale
plan or an unverified report. \emph{(II) Effect propagation} pushes the consequences (delays,
freed resources, violated guards) forward along effective dependencies, computing the minimal set of
downstream transitions that must change. \emph{(III) Bounded reassignment} re-times or re-assigns only
those transitions, under $\delta_{\max}$ and $K$, and emits the result as a proposal package. Each
phase is contained to the dependency closure of the disruption, which is what keeps the edit radius
small; the package then passes ordinary admission before any of it commits.

\begin{proposition}[Bounded Reactive Repair]
\label{prop:brr}
Under LCRP with iteration cap $K$ and displacement budget $\delta_{\max}$, repair time satisfies
$T_{\mathrm{repair}}\in\mathcal{O}\!\big(K(N_aM+N_aO_{\max})\big)$ and the cumulative displacement
satisfies $\Delta(\sigma^\ast)\le\delta_{\max}$, where $N_a$ is the number of affected entities, $M$ the
number of resources, and $O_{\max}$ the maximum transitions per entity. Repair terminates within $K$
revalidation rounds; if either bound is exceeded, the runtime escalates deterministically by rolling
back to the most recent restore point. No transition commits while its violation set against
$\mathcal{C}$ is non-empty.
\end{proposition}
\noindent\emph{Proof idea: each round validates and propagates in bounded work, rounds are capped at $K$, the budget is checked per round, and commit requires an empty violation set; escalation lands on a validated restore point (Proposition~\ref{prop:rollback}); full proof in the supplement \cite{Mnemosyne-Supp}.}

\begin{proposition}[Transactional Rollback Safety]
\label{prop:rollback}
Let $S_j$ be the most recent restore point. Rollback guarantees that (i) no infeasible state remains
committed, (ii) all effects after $S_j$ are logically undone through dependency-closed compensation,
and (iii) replay from $S_j$ is deterministic.
\end{proposition}
\noindent\emph{Proof idea: restore points are post-admission, log entries are immutable, and idempotent dependency-closed compensation undoes effects after $S_j$; full proof in the supplement \cite{Mnemosyne-Supp}.}

These guarantees are relative to $\mathcal{C}$ and the runtime's effective-state model; they do not
assert that the proposer finds a globally optimal repair, nor that $\mathcal{C}$ is semantically
complete.

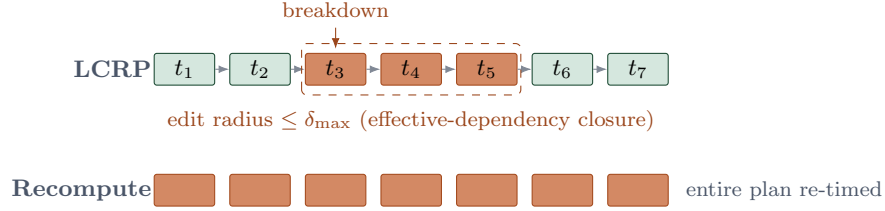
\begin{figure}[t]
\centering
\adjustbox{max width=\linewidth}{%
\begin{tikzpicture}[font=\footnotesize,
  hit/.style={draw=atpamber!85!black,fill=atpamber!70,rounded corners=1pt,minimum width=0.8cm,minimum height=0.42cm,inner sep=1pt},
  fix/.style={draw=atpgreen!55!black,fill=atpgreen!20,rounded corners=1pt,minimum width=0.8cm,minimum height=0.42cm,inner sep=1pt},
  dep/.style={-{Latex[length=4pt]},atpslate!70,thin}]
\node[anchor=east,atpslate] at (-0.3,0) {\textbf{LCRP}};
\node[fix](a1) at (0,0){$t_1$};
\node[fix](a2) at (1.0,0){$t_2$};
\node[hit](a3) at (2.0,0){$t_3$};
\node[hit](a4) at (3.0,0){$t_4$};
\node[hit](a5) at (4.0,0){$t_5$};
\node[fix](a6) at (5.0,0){$t_6$};
\node[fix](a7) at (6.0,0){$t_7$};
\foreach \i/\j in {a1/a2,a2/a3,a3/a4,a4/a5,a5/a6,a6/a7}\draw[dep](\i)--(\j);
\draw[-{Latex[length=4pt]},atpamber!80!black] (2.0,0.55)--(2.0,0.27);
\node[anchor=south,atpamber!80!black] at (2.0,0.55){\scriptsize breakdown};
\draw[dashed,atpamber!80!black,rounded corners=2pt] (1.55,-0.34) rectangle (4.45,0.34);
\node[anchor=north,atpamber!80!black] at (3.0,-0.40){\scriptsize edit radius $\le\delta_{\max}$ (effective-dependency closure)};
\node[anchor=east,atpslate] at (-0.3,-1.6) {\textbf{Recompute}};
\foreach \k in {0,1,2,3,4,5,6}\node[hit] at (\k*1.0,-1.6){};
\node[anchor=west,atpslate] at (6.5,-1.6){\scriptsize entire plan re-timed};
\end{tikzpicture}}
%\caption{Bounded reactive repair vs.\ global recompute on a disrupted plan (illustrative; the $5\times3$ job-shop instance of \S\ref{sec:rq6} moves from makespan $19$ to $22$). A breakdown at $t_3$ confines LCRP to the effective-dependency closure $\{t_3,t_4,t_5\}$ (amber) within the displacement budget $\delta_{\max}$, leaving the rest committed (green); global recompute re-times every transition. In the running example, $t_3$ is the closed airport leg, $t_4,t_5$ are the pickups that depend on it, and the green transitions are the dinner tasks a global re-plan would needlessly reshuffle. RQ6 (\S\ref{sec:rq6}) measures this edit-radius gap.}
\caption{Bounded reactive repair vs.\ global recompute (illustrative). A breakdown at $t_3$
confines LCRP to the effective-dependency closure $\{t_3,t_4,t_5\}$ (amber), leaving the rest
committed (green); global recompute re-times every transition. RQ6 (\S\ref{sec:rq6}) measures
the edit-radius gap.}
\label{fig:lcrp-containment}
\Description{A two-row schematic. The top row, labeled LCRP, shows seven transitions in a dependency
chain; a breakdown at the third transition confines edits to a dashed box around the third through
fifth transitions, while the others remain fixed. The bottom row, labeled Recompute, shows all seven
transitions edited.}
%\vspace{-.1in}
\end{figure}

\paragraph{Recovery during recovery.}
A repair is itself an execution, so a failure can strike \emph{while a repair is in flight}: a second
disruption arrives before the first is absorbed, or a compensation itself fails. This recursive
recovery appears to demand a recovery stack of tracked, bounded depth. It does
not, but the reason must be stated carefully, because a disruption is a new fact about the world and
must never be dropped. After the most recent committed restore point $S_k$, the runtime maintains a
\emph{pending observation set} $D_k$: the disruptions observed since $S_k$ that no admitted repair has
yet absorbed. Every repair candidate is validated against $S_k$, $\mathcal{C}$, and \emph{all} of
$D_k$, so it must account for every outstanding failure, not only the latest.

%\vspace{-.06in}
\begin{lemma}[Memorylessness with respect to recovery history]
\label{lem:memoryless}
Because admission reads only the latest committed restore point $S_k$, the constraint set
$\mathcal{C}$, and the accumulated observation set $D_k$, a repair candidate is a function
$R(S_k,\mathcal{C},D_k)$. It is memoryless with respect to \emph{recovery history} (independent of whether earlier repair attempts were in progress, aborted, or superseded), but \emph{not} with respect
to world observations, since $D_k$ retains every disruption not yet absorbed into an admitted repair.
\end{lemma}

%\vspace{-.1in}
\begin{proposition}[Recursive recovery reduces to sequential recovery]
\label{prop:rrr}
If every in-flight repair remains an uncommitted proposal and every admitted repair is validated
against $S_k$, $\mathcal{C}$, and the full pending set $D_k$, then a failure arriving mid-repair is
equivalent to appending it to $D_k$ and restarting single-shot recovery from $S_k$. Recursive recovery
is therefore exactly sequential recovery over the accumulated observations: nesting introduces no new
case, every committed transition is $\mathcal{C}$-feasible at any nesting depth and failure timing, and
a sequence of $F$ failures terminates in at most $FK$ rounds before success or escalation.
\end{proposition}
\noindent\emph{Proof idea: by Lemma~\ref{lem:memoryless} a candidate depends only on $(S_k,\mathcal{C},D_k)$; interrupting an uncommitted repair changes no committed state, so a mid-repair failure re-poses single-shot recovery over the enlarged $D_k$; full proof in the supplement \cite{Mnemosyne-Supp}.}

\paragraph{Why reactive repair is the hard half.}
A generated repair runs over \emph{effective} state, must preserve the evidence that
triggered it, must close its own dependencies and conflict scope, and may be interrupted by a
fresh disruption; \ATP{} nonetheless keeps every committed transition feasible and loses no
observation, with edit radius and recovery depth both bounded. Reactive planning is thus fully in
scope; only \emph{preemptive} guard synthesis is deferred to future work (see the scope statement
closing \S\ref{sec:conclusion}).

\section{Mnemosyne Architecture and System Specification}
\label{sec:arch-spec}

\Mnemo{} is the executable substrate for ATP. Its design objective is not to make the proposer smarter; it is to make proposal authority explicit. The architecture therefore separates four planes: the \emph{proposal plane}, where LLMs, solvers, benchmark adapters, learned policies, and active contracts generate candidates; the \emph{admission plane}, where deterministic validators decide whether a candidate may become truth; the \emph{commit plane}, where admitted transitions are durably appended; and the \emph{projection plane}, where current effective state is materialized for future admission. Runtime engines sit outside these planes as orchestrators. They schedule work, retry steps, and detect possible contract readiness, but they do not own committed truth. Figure~\ref{fig:mnemo-architecture} shows these planes, the proposers that feed them, and the durable stores that only admitted transitions may extend.

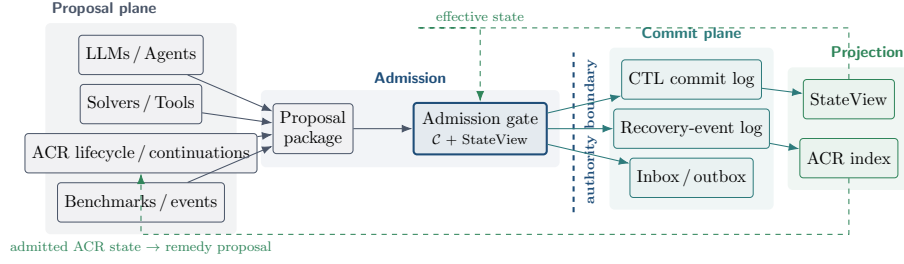
\begin{figure*}[t]
%\vspace{-.06in}
\centering
\adjustbox{max width=0.858\textwidth}{%
\begin{tikzpicture}[font=\small]
\node[proposer] (p1) at (0,1.5) {LLMs\,/\,Agents};
\node[proposer] (p2) at (0,0.5) {Solvers\,/\,Tools};
\node[proposer] (p3) at (0,-0.5) {ACR lifecycle\,/\,continuations};
\node[proposer] (p4) at (0,-1.5) {Benchmarks\,/\,events};
\node[proposer] (pkg) at (3.5,0) {Proposal\\package};
\node[gate] (adm) at (6.9,0) {Admission gate\\\scriptsize $\mathcal{C}$ + StateView};
\draw[atpblue, very thick, dashed] (8.8,-1.7) -- (8.8,1.55);
\node[atpblue, font=\scriptsize\bfseries, rotate=90] at (9.15,-0.05) {authority boundary};
\node[store] (ctl) at (11.2,1.0) {CTL commit log};
\node[store] (rec) at (11.2,0.0) {Recovery-event log};
\node[store] (io) at (11.2,-1.0) {Inbox\,/\,outbox};
\node[view] (sv) at (14.4,0.6) {StateView};
\node[view] (acr) at (14.4,-0.6) {ACR index};
\begin{scope}[on background layer]
\node[fill=atpslate!7, rounded corners=3pt, fit=(p1)(p4), inner sep=2.4mm] (planeP){};
\node[fill=atpblue!6, rounded corners=3pt, fit=(pkg)(adm), inner sep=2.4mm] (planeA){};
\node[fill=atpteal!7, rounded corners=3pt, fit=(ctl)(io), inner sep=2.4mm] (planeC){};
\node[fill=atpgreen!7, rounded corners=3pt, fit=(sv)(acr), inner sep=2.4mm] (planeV){};
\end{scope}
\node[planelabel, anchor=south west] at (planeP.north west) {\color{atpslate}Proposal plane};
\node[planelabel, anchor=south] at (planeA.north) {\color{atpblue}Admission};
\node[planelabel, anchor=south] at (planeC.north) {\color{atpteal}Commit plane};
\node[planelabel, anchor=south east] at (planeV.north east) {\color{atpgreen}Projection};
\foreach \p in {p1,p2,p3,p4} {\draw[flow] (\p) -- (pkg);}
\draw[flow] (pkg) -- (adm);
\draw[commitflow] (adm) -- (ctl);
\draw[commitflow] (adm) -- (rec);
\draw[commitflow] (adm) -- (io);
\draw[commitflow] (ctl) -- (sv);
\draw[commitflow] (rec) -- (acr);
\draw[feedback] (sv.north) |- (8.0,2.05) -- (5.6,2.05) -| (adm.north)
   node[pos=0.5,above,font=\scriptsize,atpgreen]{effective state};
\draw[feedback] (acr.south) |- (8.0,-2.15) -- (3.0,-2.15) -| (p3.south)
   node[pos=0.5,below,font=\scriptsize,atpgreen]{admitted ACR state $\rightarrow$ remedy proposal};
\end{tikzpicture}%
}
%\vspace{-.1in}
\caption{System architecture. Proposers generate packages; deterministic admission under $\mathcal{C}$ and current effective state is the only authority boundary; only admitted transitions cross it into the durable logs that feed \StateView{} and the active-contract index. Runtime drivers (not shown) orchestrate but never own committed truth.}
\Description{A block diagram shows external triggers and proposers feeding proposal packages into deterministic admission. Admission writes accepted records to the committed-transition log and recovery-event log, projects StateView, creates active contract records, and stages external intent through an inbox/outbox boundary. Runtime drivers orchestrate these paths but do not directly write committed truth.}
%\vspace{-.1in}
\label{fig:mnemo-architecture}
\end{figure*}

\textbf{Durable substrate.} The committed-transition log (\CTL{}) is the source of committed truth: tenant-scoped, versioned, and logically append-only, it stores admitted transitions, compensation and supersession records, and \ACR{} lifecycle records, so later transitions may supersede earlier ones without erasing the fact that they were admitted. \StateView{} projects \emph{current} truth from \CTL{} by replaying only effective records (not compensated, superseded, or invalidated, with dependency chains intact); admission reads \StateView{}, never raw history, which is the mechanism behind effective-state soundness. A recovery-event log records the fine-grained events of repair and obligation execution, and in the PostgreSQL substrate tenant-scoped uniqueness over event identifiers, idempotency keys, and recovery sequence positions is the concrete idempotency boundary of Theorem~\ref{thm:iso}.

\textbf{Authority boundary.} The admission gate is the only authority boundary. A proposal package carries a candidate transition, declared scope, proposer identity, an optional solver certificate, world assumptions, and evidence; the gate checks finite-state validity, idempotency, scope, dependency closure, compensation safety, stale-world conditions, conflicts, and evidence preservation, treats a solver certificate or LLM rationale as evidence rather than authority, and returns an admitted transition, a reasoned rejection, or an escalation. \ACRs{} are durable obligations in the same substrate, recording guards, triggers, continuations, compensation hooks, scopes, and validators. A trigger produces a readiness lifecycle proposal; admission must commit readiness or breach before the continuation may construct a repair proposal, which then passes ordinary admission. Memory is active but non-authoritative at both crossings. The supplement \cite{Mnemosyne-Supp} separates the implemented flat ACR core from a forward-compatible profile normalization that is not evaluated in the tagged artifact.

\textbf{Boundaries and drivers.} External events enter through an inbox deduplication boundary, and external effects leave through an outbox that stages provider calls with idempotency keys, so a domain commit never silently becomes an external effect and duplicate observations never create duplicate commits. Runtime drivers, including durable engines such as Temporal and Cadence~\cite{temporal_docs,cadence_github}, orchestrate execution, timers, retries, and detection but remain non-authoritative: a driver may submit a readiness proposal grounded in admitted evidence, and only after its admission may it call a continuation provider and submit the resulting remedy package. It may not mutate lifecycle or domain truth directly. Table~\ref{tab:component-spec} states each component's authority contract: what it owns,  does not own, and the failure that results if its boundary is bypassed.

\begin{table}[t]
\centering
\caption{Component authority contract. Each component owns one boundary; bypassing it produces the named failure.}
\label{tab:component-spec}
\footnotesize
\setlength{\tabcolsep}{3pt}
\begin{tabularx}{\linewidth}{@{}p{0.19\linewidth} >{\raggedright\arraybackslash}X >{\raggedright\arraybackslash}X >{\raggedright\arraybackslash}X@{}}
\toprule
\rowhead
Component & Owns & Does not own & Failure if bypassed \\
\midrule
LLM / solver  & Proposal generation & Committed truth & Hallucinated commit \\
Admission gate & Auth bound under $\mathcal{C}$ & Proposal generation & Invalid commit \\
\CTL{} & Durable history & Effective truth alone & Lost audit and replay \\
\StateView{} & Current effective truth & Full history & Stale, compensated state used \\
\ACR{} & Durable obligation & Mutation authority & Truth corrupted \\
Store substrate & Atomic and idempotency & Semantic validation & Duplicate or conflict anomaly \\
\bottomrule
\end{tabularx}
%\vspace{-.1in}
\end{table}

The full proposal-and-obligation lifecycle (\textsf{propose} $\rightarrow$ \textsf{admit} $\rightarrow$ \textsf{commit} $\rightarrow$ \textsf{project}, with a contract's \textsf{create} $\rightarrow$ \textsf{watch} $\rightarrow$ \textsf{readiness proposal} $\rightarrow$ \textsf{admit readiness} $\rightarrow$ \textsf{resume} rejoining the same proposal-admission path) together with the six-point functional contract, the storage substrate (PostgreSQL-backed, with a SQLite fallback for portability), and the deployable service boundary, is specified in the supplementary material \cite{Mnemosyne-Supp}. The invariant they enforce is the one already proved here: neither a readiness detector nor an \ACR{} continuation is a commit authority, so no component receives a privileged path because it is a solver, a learned guard, or an active contract.

\section{Evaluation}
\label{sec:eval}

We evaluate ATP as a data-management system, not as an LLM
planning method. The experiments validate the assumptions behind the theorems and
the correctness of the implementation; they make no claim about learning efficacy, later
outcome improvement, or preemptive planning. A public cross-episode scoring harness for those questions
now exists (the REALM-Bench Tier~6 extension used in RQ3), and their confirmatory evaluation is
a registered, ongoing study; this paper claims only boundary behavior. Planning and recovery workloads are
derived from the REALM-Bench planning benchmark and instantiated here as the
J1--J4 benchmark harness~\cite{realmkdd2026}. We reserve the name P1--P10 for
broader readiness suites that are tracked as a separate artifact and are not
certified in this draft. The evaluation is organized around six research questions, sorted into the paper's two
movements. \emph{The gate under attack:} RQ1 asks whether the authority boundary holds across
hazard classes and end-to-end; RQ2 compares against workflow/saga guardrails; RQ3 asks whether
the same boundary holds under \emph{live}, heterogeneous LLM proposers, at static plan entry and
under mid-execution disruption, with all metrics read by an external scorer from boundary
traces. \emph{Usefulness, cost, and the reaction:} RQ4 separates proposer usefulness from
committed-state correctness; RQ5 audits artifact-level cost; RQ6 examines reactive repair, how
contained a bounded repair is, and whether recovery during recovery stays safe. RQ1, RQ2, RQ4,
RQ5, and RQ6's controlled interruption stress are completed in the versioned \Mnemo{} artifact;
RQ3 reports bounded live-proposer pilots. RQ6's containment (edit radius) is reported on a companion job-shop scheduling harness
that implements the same bounded-repair protocol, where blast radius is directly measurable
(\S\ref{sec:rq6}); the claim is a systems property (bounded edit radius at equal feasibility),
not planning quality. Recovery during recovery closes by reduction (Proposition~\ref{prop:rrr})
and is separately exercised by controlled interruption at four checkpoints and depths one through
three (\S\ref{sec:rq6}).

\subsection{Overview and Scoreboard}
The versioned artifact covers the falsification-oriented empirical spine: six theorem- and
property-validation benchmarks, one per hazard class, two end-to-end usefulness benchmarks, one mechanism-level workflow/saga guardrail
comparator, and one infrastructure-cost audit. Separate extensions add an executable Temporal check,
a data-size scaling sweep, and recursive-recovery stress. The full benchmark roster and per-benchmark
designs are in the supplement \cite{Mnemosyne-Supp}.

These benchmarks are falsification-oriented systems tests rather than stochastic
model-quality experiments.  The deliberately permissive baselines
expose each targeted failure mode by giving generated actions, repairs, wakeups,
or storage writes more authority than ATP permits. A result is meaningful only when the comparator commits the specified violation. Figure~\ref{fig:violation-suppression} summarizes the safety evidence across every hazard class and the comparator; the cost audit (RQ5) is reported separately because it measures infrastructure overhead, not violation suppression.

%\vspace{-.1in}
\subsection{RQ1: Does the Authority Boundary Hold Across Hazard Classes and End-to-End?}
\label{sec:eval-safety-ref}
The first and primary question is whether the admission boundary actually holds: across every hazard class ATP is designed to stop, does committed state stay $\mathcal{C}$-valid? We answer it with six falsification benchmarks, one per hazard class, each comparing ATP against deliberately permissive baselines. Figure~\ref{fig:violation-suppression} is the scoreboard across all nine benchmarks: in every class the permissive baseline commits the targeted violation, while ATP commits zero such violations and still admits the valid operations. Per-class designs and full result tables are in the supplementary material \cite{Mnemosyne-Supp}; we summarize the six classes here.

\begin{figure}[th]
\centering
\begin{tikzpicture}
\begin{axis}[
  xbar, clip=false, width=\linewidth, height=5.2cm,
  bar width=6.8pt, enlarge y limits=0.10,
  xmin=0, xmax=80,
  xlabel={\footnotesize Targeted violations committed (lower is safer)},
  symbolic y coords={Mechanism,Proposer,EndToEnd,Storage,Projection,Obligation,Evidence,Serial,Authority},
  ytick=data, tick label style={font=\scriptsize},
  nodes near coords, nodes near coords align={horizontal},
  every node near coord/.append style={font=\tiny},
  axis x line=bottom, axis y line=left,
  x axis line style={atpslate}, y axis line style={atpslate},
  tick align=outside, xmajorgrids, grid style={atpslate!20},
  legend style={at={(0.98,0.97)}, anchor=north east, draw=atpslate!40, font=\scriptsize},
  legend image code/.code={\draw[#1] (0cm,-0.06cm) rectangle (0.22cm,0.11cm);},
]
\addplot[draw=atpamber, fill=atpamber!80,
  every node near coord/.append style={text=atpamber}] coordinates
  {(35,Authority) (48,Serial) (20,Evidence) (16,Obligation) (7,Projection) (64,Storage) (3,EndToEnd) (37,Proposer) (6,Mechanism)};
\addplot[draw=atpgreen, fill=atpgreen!80,
  every node near coord/.append style={text=atpgreen, anchor=west, xshift=1pt}] coordinates
  {(0,Authority) (0,Serial) (0,Evidence) (0,Obligation) (0,Projection) (0,Storage) (0,EndToEnd) (0,Proposer) (0,Mechanism)};
\legend{unsafe baseline, ATP\,/\,\Mnemo{}}
\end{axis}
\end{tikzpicture}
%\vspace{-.2in}
\caption{Violation suppression across the nine mechanism benchmarks. Amber bars show the most permissive unsafe baseline for each of the six hazard classes, the two end-to-end benchmarks, and the 14-case workflow/saga mechanism audit; green bars show \ATP{}. The separate Temporal SDK check appears in Table~\ref{tab:rq2-guardrail-results}.}
\Description{A horizontal bar chart shows baseline violation counts across nine mechanism benchmarks as 35, 48, 20, 16, 7, 64, 3, 37, and 6, while the ATP violation count is zero for every benchmark.}
\label{fig:violation-suppression}
%\vspace{-.15in}
\end{figure}
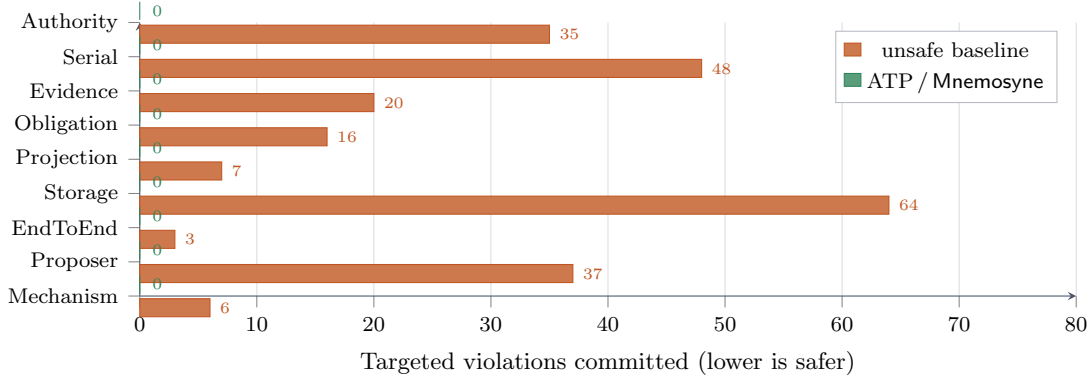

\textsc{AuthorityBench} (Theorem~\ref{thm:idc}) gives five proposer classes direct or self-validated commit authority: both baselines commit $35$ invalid transitions, while ATP commits zero and still admits the five valid ones. \textsc{SerialAdmissionBench} (Theorem~\ref{thm:iso}) runs $80$ concurrent proposals over a shared capacity object; the unserialized and weak-lock baselines commit invalid proposals and underflow capacity, while ATP commits only the $32$ valid proposals through a serialized boundary and yields a serial-equivalent history. \textsc{ObligationBench} (Theorem~\ref{thm:oc}) exercises active contract records under four continuation sources; the baselines let wakeups mutate domain state directly ($12$ and $16$ unauthorized mutations), while ATP admits the lifecycle transition, routes each continuation output through ordinary admission, and commits zero unauthorized domain mutations. \textsc{CompensationProjectionBench} stresses compensation, supersession, and replay across seven dependency scenarios; the projection baselines admit invalid compensations and project ineffective history as current truth (seven \StateView{} mismatches), while ATP keeps \StateView{} the projection of effective committed records only, with zero mismatches. \textsc{StorageSubstrateBench} exercises duplicate, stale, and malformed storage attempts; the unconstrained log commits $64$ invalid attempts and mis-projects the total, while the ATP storage path (PostgreSQL-backed; a SQLite fallback remains for portability) rejects all $64$ and preserves projection.

The distinctive case is \textsc{EvidenceRepairBench} (Theorem~\ref{thm:epr}). It constructs repairs that discharge their own trigger by deleting a failing test, suppressing an alert, compensating a diagnostic record, or hiding a violated guard. A naive repair agent and a workflow baseline without an evidence-preservation rule each commit $20$ such evidence-destroying repairs; ATP commits zero, rejecting every repair that would succeed merely by removing the signal that justified it, while still admitting the eight genuinely valid repairs. This is the property that separates ATP from audit logging: the alarm cannot be silenced by the repair.

This is the operational content of the four theorems of Section~\ref{sec:theorems}, relative to the declared constraint set $\mathcal{C}$. Because the boundary is enforced by construction, these results are best read as evidence that the implemented gate is wired and unbypassed on the tested hazards, not as a competitive comparison; the comparison against state-of-practice mechanisms is RQ2.

\smallskip\noindent\textbf{End-to-end.} Beyond the unit hazards, real planning-and-recovery
cases must run end-to-end through the same boundary. J1--J4 instantiate four REALM-Bench
planning and disruption-recovery cases~\cite{realmkdd2026}, the benchmark family from which the
Thanksgiving reunion of \S1 is drawn; the invalid packages include stale-world and
dropped-commitment faults of exactly the kind that opened this paper, a reused off-peak drive
time and a silently vanished pickup. Each case emits valid and invalid
packages (six valid, three invalid) along benchmark case $\rightarrow$ proposal $\rightarrow$
admission $\rightarrow$ \CTL{} $\rightarrow$ \StateView{}, measured against a direct-workflow
baseline that commits every package. Table~\ref{tab:rq3-end-to-end-results} shows the result:
ATP admits the six valid packages, rejects the three invalid ones, completes all four cases
with zero invalid commits and a consistent \StateView{}, while the baseline commits all nine
and leaves three \StateView{} mismatches. These cases also exercise the reactive-repair path of
\S\ref{sec:lcrp}: when a disruption occurs, the affected contract's readiness transition is
admitted before its continuation proposes a bounded local repair that re-enters admission, so end-to-end safety already covers reactive
recovery; RQ6 then quantifies how \emph{contained} that repair is.

\begin{table}[t]
\centering
\caption{RQ1 end-to-end execution on J1-J4. The unsafe baseline commits invalid case proposals
and corrupts \StateView{}; ATP completes all four cases while rejecting invalid ones.}
\label{tab:rq3-end-to-end-results}
\setlength{\tabcolsep}{3.2pt}\footnotesize
\begin{tabular}{lrrrrrr}
\toprule
\rowhead
System & Pkgs. & Admit. & Reject & Invalid & Done & Mismatch \\
\midrule
Direct workflow & 9 & 9 & 0 & 3 & 4 & 3 \\
ATP / \Mnemo{} & 9 & 6 & 3 & 0 & 4 & 0 \\
\bottomrule
\end{tabular}
%\vspace{-.15in}
\end{table}

\subsection{RQ2: Does ATP Catch Hazards Workflow/Saga Guardrails Miss?}
RQ1's permissive comparators isolate failure classes. A 14-case mechanism audit crosses four valid
operations, four classical guardrail hazards, and six ATP-specific hazards through schema/FSM checks,
idempotency, timers, local compensation, and proposer self-checking. The stack rejects the four
classical hazards but commits all six semantic hazards; ATP rejects all ten invalid cases and admits
the four valid ones. Figure~\ref{fig:violation-suppression} reports this ``Mechanism'' result.

\paragraph{Executable Temporal check.}
We next ask whether the distinction survives contact with a real workflow engine. We run Temporal SDK 1.32.0~\cite{temporal_docs} with its test
server and an actual worker. Each case crosses a Temporal activity that applies schema, finite-state,
idempotency-key, and proposer-self-check guardrails, then persists accepted domain state to SQLite.
We compare that configuration with the same Temporal workflow and store plus ATP admission. The
four matched cases comprise one valid repair and three well-formed semantic hazards: a stale-world
plan, an evidence-destroying unresolved repair, and a compensation that orphans an effective
dependent.

\begin{table}[t]
\centering
\caption{RQ2 real Temporal SDK comparison. Both configurations execute the same workflows and
local guardrails; only the second adds ATP admission.}
\label{tab:rq2-guardrail-results}
\setlength{\tabcolsep}{3.5pt}\small
\begin{tabular}{lrrrrr}
\toprule
\rowhead
System & Cases & Commit & Reject & Invalid & Valid \\
\midrule
Temporal guardrails & 4 & 4 & 0 & 3 & 1 \\
Temporal $+$ ATP & 4 & 1 & 3 & 0 & 1 \\
\bottomrule
\end{tabular}
\end{table}

Table~\ref{tab:rq2-guardrail-results} reports the persisted outcomes. Temporal with conventional
guardrails completes all four workflows and commits all three semantic hazards. Adding ATP rejects
the three hazards before the domain write and retains the valid commit. Every workflow has an
11-event Temporal history. This is not a claim that Temporal is defective: Temporal supplies
durable orchestration, while ATP supplies a semantic admission contract that application code must
otherwise author. The supplement \cite{Mnemosyne-Supp} retains the broader 14-case mechanism audit covering classical
guardrail failures and six ATP-specific hazard classes.

\paragraph{Documented capability boundary.}
LangGraph documents persistence, durable execution, and configurable human interrupts
\cite{langgraph_overview_2026,langgraph_hitl_2026}; Cadence documents durable history, retries,
timers, activities, and application-authored saga compensation
~\cite{cadence_engine_2026,cadence_orchestration_2026}. Both can host ATP validators, but neither
documents ATP's five semantic contracts as built-in guarantees. The supplement \cite{Mnemosyne-Supp} audits each
capability; ``not built in'' does not mean ``impossible to implement.''

\paragraph{Production-TP substrate comparison.}
\ATP{} is not a replacement for production transaction processing. It uses TP substrates for
atomicity, durability, uniqueness, and isolation, and adds an admission boundary for generated
proposals before they become transactions. The supplementary material \cite{Mnemosyne-Supp} includes a coverage audit over
PostgreSQL-style TP mechanisms, workflow/saga guardrails, and \ATP{}. The audit shows that ordinary
TP mechanisms catch storage-level violations, while \ATP{}-specific hazards require
\StateView{}-based admission, evidence-preserving repair, non-authoritative \ACR{} wakeups,
dependency-closed compensation, and generative conflict scopes.

\subsection{RQ3: Does the Boundary Hold Under Live, Heterogeneous LLM Proposers?}
\label{sec:eval-live}

RQ1 and RQ2 validate the boundary with deterministic adapters that replay live-model failure
modes, which makes the safety results exactly reproducible; RQ4--RQ6 use the same discipline for
usefulness, cost, and containment. RQ3 closes the gap between adapters and reality: it places \emph{live} model outputs behind the same admission boundary and reads every
metric from boundary traces through an external scorer, the cross-episode (Tier~6) extension of
REALM-Bench~\cite{realmkdd2026}.\footnote{Tier~6 derives deterministic cross-episode sequences from
existing REALM-Bench families and reads boundary traces only; its specification, scorer, and the
pilot traces ship with the artifact.} Two pilots are reported, both deliberately bounded: one task
family, four proposer packs (Claude, GPT, DeepSeek expert, DeepSeek instant), and no claim of
learning efficacy or family generality.

\emph{Static plan entry.} The four packs answer one public prompt pack over a ten-episode sequence
(40 proposals), with no access to hidden failure labels, future outcomes, or an answer key. All 40
handoff cases validate; the safety gate passes with zero invalid commits, zero evidence-destroying
repairs, and zero orphaned dependents.

\emph{Mid-execution disruption.} The stricter pilot targets the reactive half of the boundary
(\S\ref{sec:lcrp}): execution has begun, operations have committed, and a disruption arrives. Ten
mid-execution episodes are posed to the same four packs (40 responses); each response proposes
\texttt{repair}, \texttt{reject}, or \texttt{observe}, and \Mnemo{} replays it through deterministic
admission guards requiring committed operations to remain unchanged, committed evidence to be
preserved, rollback scope to remain none or local, and repair radius to remain bounded. The gate
admits 24 proposals and rejects 16, four of them explicit safety rejections with structured
reasons (e.g., \texttt{repair\_radius\_exceeded:2>1}); representative admitted-repair,
safe-rejection, and refused-rollback cases are in the supplement \cite{Mnemosyne-Supp}. The safety gate passes
all 40 events. Table~\ref{tab:rq6-live} summarizes both pilots.

\begin{table}[t]
\centering
\caption{RQ3 live-proposer pilots. Both pilots use four heterogeneous LLM packs on one task family;
metrics are boundary observations computed by the external Tier~6 scorer, not read from any model.
Usefulness-side scorer outputs are reported for completeness and carry no planning-quality claim.}
\label{tab:rq6-live}
\footnotesize
\begin{tabular}{lcc}
\toprule
 & \textbf{Static entry} & \textbf{Mid-execution} \\
\midrule
Live events scored            & 40      & 40 \\
Admitted / rejected           & --      & 24 / 16 \\
Explicit safety rejections    & --      & 4 \\
Invalid commits               & 0       & 0 \\
Evidence-destroying repairs   & 0       & 0 \\
Orphaned effective dependents & 0       & 0 \\
Scorer safety gate            & pass    & pass \\
\midrule
Horizon reward (usefulness)   & 0.90    & 0.93 \\
Grounded-admission rate       & 0.73    & 0.60 \\
\bottomrule
\end{tabular}
%\vspace{-.08in}
\end{table}

The reading is architectural, and it is the paper's thesis under live conditions: the pilots do not
show that LLMs are reliable repair planners (they are not; 40\% of live repair proposals were
inadmissible), but that unreliable live proposals can be placed behind an admission-gated substrate
that preserves committed evidence, rejects unsafe rollback, bounds repair radius, and emits
auditable traces an external harness can score. A rejection here is not a system failure; a bad
proposal plus a correct pre-commit rejection is the boundary doing its job. The usefulness-side
scorer outputs (horizon reward, grounded admission) are observations at pilot scale, consistent
with the usefulness/correctness separation that RQ4 quantifies, and support no learning-efficacy or
planning-quality claim; the confirmatory cross-episode study, with its ablation grid, hypotheses,
and decision rules fixed in a registered protocol, is ongoing and out of scope here.

\subsection{RQ4: Does Proposer Quality Affect Usefulness but Not Correctness?}
How does proposer competence affect utility while ATP holds committed-state
correctness fixed? We evaluate this using \textsc{ProposerQualityBench}. The
benchmark emits 60 proposals from six proposer classes: no-intelligence, random,
rule-based, solver-like, LLM-like, and adversarial proposers. Proposal quality
controls how many attempts are valid under $\mathcal{C}$ and how much utility an
admitted attempt contributes. We compare ATP with a direct-commit baseline that
commits every attempt.

\begin{table}[t]
\centering
\caption{RQ4 system-level result. Direct commit preserves all generated attempts,
including invalid ones; ATP rejects invalid attempts while preserving the same
valid utility mass.}
\label{tab:rq4-system-results}
\setlength{\tabcolsep}{4pt}\footnotesize
\begin{tabular}{lrrrrr}
\toprule
\rowhead
System & Attempts & Commit & Reject & Invalid & Utility \\
\midrule
Direct commit & 60 & 60 & 0 & 37 & 131 \\
ATP / \Mnemo{} & 60 & 23 & 37 & 0 & 131 \\
\bottomrule
\end{tabular}
%\vspace{-.08in}
\end{table}

Table~\ref{tab:rq4-system-results} shows the safety invariant. The direct-commit
baseline commits all 60 attempts, including 37 invalid attempts. ATP admits 23
valid attempts, rejects 37 invalid attempts, and commits zero invalid attempts.
The total valid utility is the same because invalid proposals carry zero utility;
therefore the usefulness signal appears in the proposer-level acceptance and
utility profile, not in the unsafe direct-commit row. The proposer classes in this artifact are deterministic adapters rather than live calls to a frontier LLM, chosen so the safety result is reproducible; the failure modes they inject are exactly those that live models such as Claude~3.7 and GPT-4o produce in multi-agent planning~\cite{sagallm2025vldb}. Because admission is evaluated against $\mathcal{C}$ and the effective \StateView{}, the authority-boundary claim is independent of how a proposal is generated: a live-model proposer changes which proposals are \emph{useful}, not whether an invalid one can commit.

\begin{figure}[th]
\centering
\begin{tikzpicture}
\begin{axis}[
  xbar, clip=false, width=0.80\linewidth, height=4.0cm,
  bar width=6.5pt, enlarge y limits={abs=0.6},
  xmin=0, xmax=72,
  xlabel={\footnotesize Admitted utility under ATP},
  symbolic y coords={Adversarial,No intel.,Random,Rule,LLM-like,Solver-like},
  ytick=data, tick label style={font=\scriptsize},
  nodes near coords, nodes near coords align={horizontal},
  every node near coord/.append style={font=\scriptsize, text=atpblue},
  axis x line=bottom, axis y line=left,
  x axis line style={atpslate}, y axis line style={atpslate},
  tick align=outside, xmajorgrids, grid style={atpslate!20},
]
\addplot[draw=atpblue, fill=atpblue!78] coordinates
  {(1,Adversarial) (1,No intel.) (4,Random) (25,Rule) (36,LLM-like) (64,Solver-like)};
\end{axis}
\end{tikzpicture}
\caption{RQ4 proposer-quality profile under ATP: stronger proposers produce more admissible utility, while every proposer class has zero invalid commits.}
%\vspace{-.12in}
\label{fig:rq3-proposer-quality}
\end{figure}
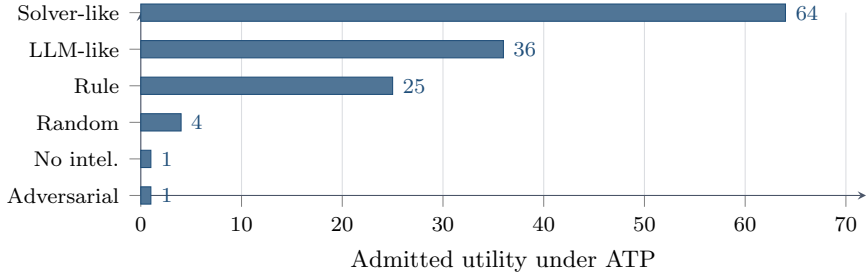

The per-proposer profile (Figure~\ref{fig:rq3-proposer-quality}; full table in the supplementary material \cite{Mnemosyne-Supp})
shows that proposer quality affects usefulness and admission efficiency. The
solver-like proposer has the highest acceptance rate and admitted utility; the
no-intelligence and adversarial proposers are mostly rejected. Yet every proposer
class has zero invalid commits under ATP. This supports the paper's central
separation: intelligence governs usefulness, while the transaction layer governs
committed-state correctness relative to $\mathcal{C}$.

\subsection{RQ5: What Does the Safety Boundary Cost?}
RQ5 separates worker-count behavior from data-size cost. First, a repository workload exercises
20 real paths---admission, commit enforcement, projection, validation, compensation, active
recovery, and the Temporal activity boundary---for 30 repeats at 1, 4, 8, and 16 process workers;
all 120 runs pass. Throughput rises from $224.5$ to $661.3$ and $768.4$ infrastructure units/min,
then falls to $303.7$ at 16 workers as local process and storage contention dominate. Thus the
artifact saturates after eight workers; we make no linear-scaling claim.

Second, a matched SQLite WAL sweep holds the 100-operation batch and CTL/\StateView{} path fixed
while adding ATP validation to one arm. At $1{,}000$, $5{,}000$, and $10{,}000$ operations, ATP
adds $6.59\%$, $5.24\%$, and $6.45\%$ throughput cost, respectively (five repeats); every run has
the expected CTL rows and zero sampled \StateView{} mismatches or validation failures. Both arms
slow as the local CTL grows, so this isolates admission cost without claiming production scale.

\subsection{RQ6: Is Reactive Repair Contained, and Safe Under Recovery-During-Recovery?}
\label{sec:rq6}
RQ1 shows reactive repair stays \emph{safe} end-to-end, and RQ3 confirms it under live proposers; here we ask how \emph{contained} it is, then whether recovery \emph{during} recovery stays safe. The
Bounded Reactive Repair guarantee (Proposition~\ref{prop:brr}) promises a repair confined to a
bounded edit radius, and containment is only meaningful on a workload where a single disruption can
ripple across many dependent operations. We therefore measure it on the job-shop instances of
REALM-Bench~\cite{realmkdd2026} using a companion scheduling harness that implements the same
bounded-repair protocol as \Mnemo{}'s reactive-repair path. After computing a feasible baseline
schedule, we inject a single machine breakdown and recover it two ways under the same disruption and
seeds: \emph{LCRP}, the bounded localized repair, edits only the affected region under the
displacement budget $\delta_{\max}$ and iteration cap $K$; \emph{global recompute} re-solves the
instance from the disruption point with no localization. Both must end feasible. Because raw makespan
alone rewards a global reset that is physically expensive on a shop floor, we score repair under a
cost-aware objective that also charges for disturbance,
\[
J = \Delta C_{\max} + \underbrace{\lambda_o N_{\mathrm{ops}} + \lambda_j N_{\mathrm{jobs}}
      + \lambda_w N_{\mathrm{wip}} + \lambda_s N_{\mathrm{setup}}}_{H},
\]
counting edited operations, touched jobs, work-in-progress moves, and machine-sequence (setup)
changes, with $\lambda_o{=}1$, $\lambda_j{=}5$, $\lambda_w{=}1$,  $\lambda_s{=}10$ as a sensitivity model
that weights job and setup disturbance above a single operation shift.

\begin{table}[t]
\centering
\caption{RQ6 disruption containment under a machine breakdown (mean and median over six job-shop instances, both feasible $6/6$): LCRP contains the disruption at lower cost $J$; global recompute lowers makespan but edits nearly an order of magnitude more of the schedule. Medians are computed per instance and are therefore not additive across columns.}
\label{tab:rq6-containment}
\setlength{\tabcolsep}{2.2pt}\footnotesize
\begin{tabular}{lrrrrrrr}
\toprule
\rowhead
Method & $\Delta C_{\max}$ & Ops & Jobs & WIP & Setup & $H$ & $J$ \\
\midrule
LCRP (mean)        & $9.50$    & $7.17$  & $1.83$ & $7.17$  & $2.50$  & $48.50$  & $\mathbf{58.00}$ \\
LCRP (median)      & $0.00$    & $3.00$  & $1.50$ & $3.00$  & $3.50$  & $53.50$  & $\mathbf{53.50}$ \\
Recompute (mean)   & $-606.17$ & $59.83$ & $8.00$ & $59.83$ & $60.83$ & $768.00$ & $161.83$ \\
Recompute (med.) & $-654.00$ & $71.00$ & $8.00$ & $71.00$ & $79.00$ & $974.50$ & $125.00$ \\
\bottomrule
\end{tabular}
%\vspace{-.08in}
\end{table}
Table~\ref{tab:rq6-containment} reports medians and means over six instances; both policies recover
feasibility ($6/6$). The contrast is the containment story behind Proposition~\ref{prop:brr}.
\emph{Global recompute wins on raw makespan} (it is free to rebuild unfinished work, so $\Delta C_{\max}$ is large and negative), \emph{but it disturbs the whole shop} to do so, editing
$59.8$ operations on average against LCRP's $7.2$ ($12.0\%$ as many, and about $4\%$ at the median),
touching all $8$ jobs against LCRP's $1.8$, and moving $59.8$ in-progress operations against LCRP's
$7.2$. LCRP holds the median makespan exactly ($\Delta C_{\max}=0$) while keeping edit radius and WIP
disturbance nearly an order of magnitude smaller, and its cost-aware objective $J$ is lower in both the mean
($58.0$ vs.\ $161.8$) and the median ($53.5$ vs.\ $125.0$). This is the empirical face of bounded
reactive repair: a disruption is contained to its neighborhood instead of triggering a global reset,
at equal feasibility while explicitly trading makespan improvement for a much smaller disturbance
radius. When the extra makespan headroom
is genuinely worth the disturbance, the runtime can still escalate to a global recompute; the point is
that ATP does not pay that cost by default.

\smallskip\noindent\textbf{Recovery during recovery.}\label{sec:rq7}
The hardest reactive case is a failure that arrives \emph{while a repair is still in flight}, and we
ask whether nesting introduces any new way to corrupt committed truth. By
construction it cannot. ATP keeps no recovery stack: after the most recent restore point $S_k$ it
maintains a pending-observation set $D_k$ of unabsorbed disruptions and validates every repair
candidate against $S_k$, $\mathcal{C}$, and all of $D_k$, so a failure arriving mid-repair is simply
appended to $D_k$ and recovery restarts single-shot from $S_k$. Proposition~\ref{prop:rrr} states the
consequence: recursive recovery is exactly sequential recovery over the accumulated observations,
every committed transition stays $\mathcal{C}$-feasible at any nesting depth and failure timing, and a
sequence of $F$ failures terminates in at most $FK$ rounds before success or escalation. A controlled
stress harness injects another failure before readiness admission, after readiness but before the
remedy, during compensation, and immediately before repair commit, at failure depths one through
three ($12$ scenarios, $K{=}3$). Admission rejects all $12$ stale in-flight candidates; the final
repairs lose zero pending observations and produce zero invalid or duplicate commits. Realized repair
rounds peak at $3$, within the largest $FK$ bound of $9$. This validates interruption and pending-set
coverage in a deterministic harness, not distributed worker failure or production load.

\subsection{Reproducibility Artifact}
\label{sec:artifact}
The released repository, tag, and commit are identified in~\cite{mnemosyne_artifact}. From the
repository root, run:
\begin{flushleft}\small
\texttt{python -m pytest -q}\\[-2pt]
\path{tests/experiments/test_rq[1-8]_*.py}\\[2pt]
\texttt{python -m pytest -q}\\[-2pt]
\path{tests/benchmarks/test_production_tp_comparator.py}\\[2pt]
\texttt{python -m pytest -q -m temporal}\\[-2pt]
\path{tests/benchmarks/test_temporal_product_baseline.py}\\[2pt]
\texttt{python -m pytest -q}\\[-2pt]
\path{tests/benchmarks/test_recursive_recovery_stress.py}\\[2pt]
\texttt{python -m pytest -q}\\[-2pt]
\path{tests/benchmarks/test_rq5_data_scale.py}\\[2pt]
\texttt{python -m benchmarks.realm.rq5\_data\_scale}
\end{flushleft}
The five test commands yield 8, 2, 1, 1, and 1 passing tests; the final command regenerates the
full RQ5 scale table in the supplement \cite{Mnemosyne-Supp}. 
Live conformance uses
\path{MNEMOSYNE_POSTGRES_DATABASE_URL}.
Reported infrastructure numbers use
PostgreSQL; the deterministic safety path does not. 

A companion supplement collects the material deferred here: proofs of the safety theorems; the
expanded failure model, benchmark roster, and per-property result tables; per-proposer-class,
concurrency, and coverage-audit detail; architecture, constraint-set, and staged-implementation
specifications; worked examples and an end-to-end transaction trace; and the RQ3 live-pilot detail
with representative cases. It also records the two-tier RQ2 audit, full RQ5 scale results, and the
RQ6 interruption matrix.

%\vspace{-.1in}
\section{Related Work}

\paragraph{\textbf{Transactions, recovery, and input validation.}} A classical transaction provides the ACID guarantees (atomicity, consistency, isolation, and durability) over a submitted unit of work, together with concurrency control, logging, and recovery, with ARIES the canonical write-ahead-logging design~\cite{gray1993transaction,bernstein2009principles,weikum2001transactional,mohan1992aries}. ATP is complementary: ACID governs how an admitted unit of work commits, whereas ATP governs whether an untrusted generated action should be admitted as a transaction at all. Schema checks and integrity constraints reject malformed \emph{commands} from a trusted application, whereas ATP distrusts the \emph{generator} itself and adds the effective-state witness, evidence preservation, obligation containment, and serial-equivalent generative concurrency, none of which a per-command predicate expresses.

\paragraph{\textbf{Compensation and event sourcing.}} Sagas decompose long-running transactions into compensatable subtransactions~\cite{garciamolina1987sagas}, and event sourcing and CQRS separate an append-only history from read-side projections~\cite{fowler2005eventsourcing,fowler2011cqrs}, themselves materialized views maintained incrementally from a change log~\cite{gupta1995maintenance}. Mnemosyne adopts both separations but adds constraints they do not: compensation must be dependency-closed over effective state and must itself pass admission, and \StateView{} is the effective-state witness the gate reads rather than a mere read model, so staleness is an admission-time verdict instead of a tolerated divergence bound~\cite{huang1994divergence,olston2002best}, since ATP targets generated proposals that may be hallucinated, stale, or adversarial rather than commands already granted authority.

\paragraph{\textbf{Active databases and provenance.}} Active database systems attach event-condition-action rules to database events~\cite{paton1999active}; \ACRs{} are deliberately weaker, waking a continuation and constructing a proposal but unable to mutate committed truth, which is the basis of obligation containment. Provenance and scientific-workflow systems record lineage for audit~\cite{buneman2001why,simmhan2005survey,ludascher2006scientific,freire2008provenance}, and reenactment reconstructs transaction provenance after the fact~\cite{arab2018reenactment}; ATP instead promotes provenance-like evidence into the admission condition itself (\EPR{}), making evidence part of the safety contract rather than only explanatory metadata.

\paragraph{\textbf{Workflow engines and LLM agents.}} Workflow engines provide retries, timers, task graphs, and durable execution~\cite{russell2005workflow,argo2023workflows,temporal_docs,cadence_github,langgraph2024}, and LLM-agent research improves generation, tool use, reflection, memory, and multi-agent simulation~\cite{yao2023react,yao2023treeofthoughts,shinn2023reflexion,zhou2023latreesearch,wang2023voyager,park2023generativeagents,wang2023surveyllmagents}. ATP treats both as proposers or drivers: engines may schedule, retry, and wake, and an agent's reasoning trace, tool call, or memory retrieval is evidence, but neither owns committed truth, which changes only through admission.

\paragraph{\textbf{Relationship to SagaLLM.}}
SagaLLM~\cite{sagallm2025vldb} and ATP pursue trustworthy multi-agent LLM execution from
opposite ends. SagaLLM brings transactional discipline \emph{to the agent}, wrapping
multi-agent planning in the saga pattern with persistent memory, validation agents, and
compensating transactions. ATP works from the \emph{data-management end}: a generated
action holds no authority until a deterministic gate admits it, and committed-state
correctness is provably independent of the proposing layer
(Theorems~\ref{thm:idc}--\ref{thm:oc}). The two are complementary: one disciplines the
proposer, the other lets the transaction layer accept untrusted proposers under one
admission boundary. On the failures SagaLLM documents, the contrast is concrete: a
canceled flight that leaves a hotel active is a \emph{dependency-closed compensation}
the gate refuses if it would orphan an effective dependent; a replan that drops a
satisfied constraint is a \emph{stale-world} rejection against the effective-state
witness; and a repair that hides the violated constraint is blocked by
\emph{evidence-preserving repair}. Beyond this duality, ATP contributes the
effective-state projection, the safety theorems at the transaction layer, and a
complete runtime; preventive guard synthesis and confirmed self-improving recovery remain future
work, though the cross-episode measurement path they require now exists and is exercised at pilot
scale in \S\ref{sec:eval-live}.

\paragraph{\textbf{Companion work.}} TRACE~\cite{chang2026trace} specifies a typed,
versioned schema for recording reasoning and commitment evidence; ATP is complementary, governing
which generated proposals may enter committed state at all. An arXiv paper~\cite{chang2027worldmodels} treats world-model synchronization as materialized-view refresh
for agentic commitments and cites this paper's admission boundary; it shares no results with this
paper.

%\vspace{-.1in}
\section{Conclusion}
\label{sec:conclusion}

Agentic workflows need transaction systems that do not trust generated actions: an
LLM may propose, but a proposal is not committed truth.
Our central claim is authority separation: relative to a declared executable
constraint set $\mathcal{C}$ and a gate-closure assumption, proposer behavior affects
proposal usefulness, not the ability to append committed truth. We formalize four
admission-boundary properties---authority separation, serial-equivalent generative admission,
evidence-preserving repair, and obligation containment---and add reactive local repair, whose
proposals re-enter the same gate. \Mnemo{} implements the resulting boundary. Across nine safety
benchmarks, \Mnemo{} rejects specified violations while admitting
valid work; a matched scale sweep measures $5.2$--$6.6\%$ incremental throughput cost for ATP
admission over the same local durable commit path, while also exposing SQLite and process-level
saturation. In a companion scheduling harness,
bounded local repair edits nearly an order of magnitude fewer operations than global recompute.
Live-proposer pilots (\S\ref{sec:eval-live}) show the boundary holding under four
LLMs at plan entry and mid-execution, scored externally with zero invalid commits.
Production deployment, family-generalized live-LLM evaluation, broader third-party
benchmarking, cross-episode learning, outcome and error analysis, semantic-impact detection, trigger
synthesis, and preemptive planning are future
work~\cite{chang2026pathagi2, chang2026trivium}.
The guarantees hold relative to $\mathcal{C}$: hazards outside it pass undetected,
and authoring a faithful $\mathcal{C}$ is engineering work; gate closure is enforced
by construction and tests, not a machine-checked proof; the Temporal comparison
(RQ2) uses an SDK test server and four controlled cases, not a production deployment; and the cost audit (RQ5) is
local, not a production load test.
ACID protects admitted transactions; ATP protects admission into committed truth.

% balance (two-column only, removed for single-column)

\clearpage
\begin{center}
{\LARGE\bfseries Supplementary Material}\\[6pt]
{\large Mnemosyne: Agentic Transaction Processing for Validating and Repairing AI-generated Workflows}
\end{center}
\medskip
\setcounter{section}{0}
\setcounter{table}{0}
\setcounter{figure}{0}
\renewcommand{\thesection}{S\arabic{section}}
\renewcommand{\thetable}{S\arabic{table}}
\renewcommand{\thefigure}{S\arabic{figure}}
\startcontents[supp]
\begingroup
\subsection*{Contents of the Supplement}
\small
\makeatletter
% widen the section-number box so S10--S22 do not collide with titles
\renewcommand*\l@section[2]{%
  \ifnum \c@tocdepth >\z@
    \addpenalty\@secpenalty
    \addvspace{1.0em \@plus\p@}%
    \setlength\@tempdima{2.6em}%
    \begingroup
      \parindent \z@ \rightskip \@pnumwidth
      \parfillskip -\@pnumwidth
      \leavevmode \bfseries
      \advance\leftskip\@tempdima
      \hskip -\leftskip
      #1\nobreak\hfil \nobreak\hb@xt@\@pnumwidth{\hss #2}\par
    \endgroup
  \fi}
\makeatother
\printcontents[supp]{l}{1}{\setcounter{tocdepth}{1}}
\endgroup
\medskip
This supplement contains the proofs, full result tables, engineering specifications, and worked examples referenced by the main paper. Theorems, sections, and figures of the main paper are cited by name; items defined here are numbered with an ``S'' prefix (e.g., Table~S1).

\section{Proofs of the Safety Theorems}
\label{s:proofs}
We restate, by name, the four safety theorems of the main paper and give their proofs. All claims are relative to the trusted, declared constraint set $\mathcal{C}$.

\begin{proof}[Authority-Separation Theorem]
We prove the claim by induction on the length of \CTL{}. Initially, \CTL{} is empty or contains only trusted initial records, so every committed transition is $\mathcal{C}$-valid. Assume after $k$ appends that every transition in $\CTL_k$ is $\mathcal{C}$-valid. Consider the $(k+1)$-st committed transition. By assumption, no LLM, solver, agent, \ACR{} watcher or continuation, learned policy, or adapter can append to \CTL{} directly; such components can only emit a proposal $p$. Since the only operation that extends \CTL{} is the append of an admitted transition, $p$ can become the next committed transition only if the gate admits it under $\mathcal{C}$ and the current effective state. By the soundness of the gate with respect to $\mathcal{C}$, the appended transition is $\mathcal{C}$-valid. Thus every transition in $\CTL_{k+1}$ is $\mathcal{C}$-valid. By induction, all committed transitions are $\mathcal{C}$-valid. The argument never assumes that the intelligence layer is accurate, rational, complete, honest, or improving; therefore committed-state correctness is decoupled from that layer.
\end{proof}

\begin{proof}[Serial-Equivalent Generative Admission Theorem]
Assume that declared conflict scopes are correct and complete for every incompatibility relevant to $\mathcal{C}$. For each such scope, admission imposes a total order on accepted proposals: if two proposals contend for the same scope, the gate admits at most one before the other observes the updated effective state or is rejected as conflicting. Because append is atomic, each admitted proposal contributes one indivisible committed transition. Interleavings across disjoint scopes commute with respect to the observations made through \StateView{}, so the global committed history is equivalent to an interleaving of the per-scope serial orders. A dependent recovery proposal is evaluated only after its dependencies appear in \CTL{} and after \StateView{} reflects the effective projection of committed records. Rejected or in-flight proposals never appear in this projection. Thus each recovery step can be explained by a serial prefix of committed history, and recovery decisions are independent of speculative attempts that never committed. Hidden conflicts outside the declared scopes are outside the guarantee.
\end{proof}

\begin{proof}[Evidence-Preserving Repair Safety Theorem]
The trigger is evaluated over committed evidence in \CTL{} and effective evidence in \StateView{}. Suppose a repair $r$ is admitted and discharges the trigger. If $r$ leaves $E$ effective and queryable, then the trigger has not been discharged by destroying evidence. Otherwise, $r$ deletes, compensates, or obscures $E$. By the evidence-preservation invariant, such a repair is admissible only when the gate verifies that the underlying failure condition $F$ is resolved under $\mathcal{C}$. Hence no admitted repair can succeed solely by removing the observable signal of $F$: it either preserves the evidence or resolves the condition. Since \CTL{} is append-only, the evidence and the gate decision remain auditable after the fact.
\end{proof}

\begin{proof}[Obligation Containment Theorem]
Creation of an \ACR{} is an admitted committed transition. When a trigger condition is
detected, the watcher constructs a readiness or breach lifecycle proposal grounded in
admitted evidence; it does not edit lifecycle state. Only if admission commits that
transition may the continuation resume and construct a remedy proposal package. That
package must pass ordinary admission before any domain transition can be appended.
Therefore an \ACR{} has no direct mutation path to \CTL{} or \StateView{} at either
crossing. It can influence committed truth only through transitions admitted by the
same gate used for every other proposer. A future obligation may become active, but it
cannot corrupt.
\end{proof}

\subsection*{Proofs of the Reactive-Repair Propositions (main paper, Section 3)}
\paragraph{Proposition (Bounded Reactive Repair).}
\begin{proof}
Each repair round validates the affected region in $\mathcal{O}(JM)$ and propagates effects along
effective dependencies in $\mathcal{O}(JO_{\max})$; rounds are capped at $K$, giving the time bound.
The displacement budget is checked each round and escalation triggers as soon as it is exceeded, so
cumulative displacement never passes $\delta_{\max}$. Commit occurs only on an empty violation set
(admission), so no infeasible transition is appended; if no in-budget repair is found within $K$
rounds, rollback restores the last validated state.
\end{proof}
\paragraph{Proposition (Transactional Rollback Safety).}
\begin{proof}
Restore points are created only after admission (empty violation set), and committed-transition-log
entries are immutable; superseded branches are excluded from the effective \StateView{} but retained
for audit. Compensation handlers keyed by idempotency keys logically undo effects after $S_j$ under
the dependency-closure rule, so rollback returns to a previously admitted, $\mathcal{C}$-feasible
state from which replay is deterministic.
\end{proof}
\paragraph{Proposition (Recursive recovery reduces to sequential recovery).}
\begin{proof}
By the main paper's Memoryless Repair Lemma, a candidate depends only on
$(S_k,\mathcal{C},D_k)$, not on the
in-flight repair it interrupts; since in-flight repairs are uncommitted, interrupting one changes no
committed state. Appending the new disruption to $D_k$ thus yields the same problem instance as a
fresh single-shot recovery from $S_k$ against the enlarged $D_k$. Safety and termination follow from
the main paper's Bounded Reactive Repair Proposition applied to each single-shot
recovery, and the $FK$ bound from at most $F$
such recoveries, each capped at $K$ rounds.
\end{proof}

\subsection*{Corollaries of the Authority-Separation Theorem (main paper, Section 3.1)}
\begin{corollary}[Solver Non-Authority]
A solver certificate proves only that a solver produced a plan under its own assumptions. It does not make the plan admissible. The plan remains a proposal until ATP admission accepts it under $\mathcal{C}$.
\end{corollary}

\begin{corollary}[Effective-State Soundness]
If \StateView{} is derived only from non-compensated, non-superseded, dependency-valid committed records, then every \StateView{} fact corresponds to an effective committed chain, and no compensated or superseded record appears as current truth.
\end{corollary}

\begin{corollary}[Dependency-Closed Compensation Safety]
If compensation is admitted only when it preserves dependency closure and effective-chain continuity, then no admitted compensation can leave an effective dependent orphaned or produce an effective state lacking a valid committed chain.
\end{corollary}

\section{Expanded Failure Model}
\label{sec:failure-model}
Mnemosyne addresses failure modes that ACID does not govern, because ACID constrains a unit of work only after it is admitted. We describe the principal modes below; Table~\ref{tab:failures-mechanisms} pairs each with the mechanism that rejects or contains it.

\textbf{Invalid generated transition.} An LLM proposes a workflow transition that violates a finite-state constraint. A conventional workflow engine may persist the action unless application code rejects it. ATP requires a deterministic validator to reject it before commit.

\textbf{Compensation with effective dependents.} A record is compensated after downstream records have depended on it. Raw log history still contains the record, but it is no longer effective. ATP disallows compensation that would orphan currently effective dependents or break the effective-state chain.

\textbf{Conflicting proposals.} Two agents propose different plans for the same tenant and entity. Both may be individually feasible. ATP rejects the active proposal set before commit because the proposals conflict over the same admission scope.

\textbf{Stale-world proposal.} A solver proposes a route assuming a deadline of 17:00, but an observed world snapshot reports the deadline is now 11:00. ATP rejects the proposal before commit because its assumptions no longer match observed facts.

\textbf{Evidence-destroying repair.} A coding agent observes that continuous integration is red and proposes to ``repair'' the repository by deleting the failing tests. A periodic checker may see the next CI run turn green and conclude the violation is gone, but the defect has not been repaired: the evidence has been destroyed. ATP requires admission to preserve the evidence that triggered repair unless the validator accepts that the repair addresses the underlying condition rather than suppressing its observable signal.

\begin{table}[t]
\centering
\caption{Generated-workflow failures and ATP mechanisms.}
\label{tab:failures-mechanisms}
\begin{tabular}{p{0.42\linewidth}p{0.48\linewidth}}
\toprule
\rowhead
Failure & ATP mechanism \\
\midrule
Invalid transition & deterministic admission \\
Unsafe compensation & dependency-closed compensation \\
Conflicting proposals & scoped proposal-conflict rejection \\
Stale assumptions & stale-world rejection \\
Evidence-destroying repair & evidence-preserving admission \\
Unbounded live obligation & active contract lifecycle \\
Duplicate concurrent repair & idempotent recovery-event admission \\
\bottomrule
\end{tabular}
\end{table}

\section{Additional Safety Invariants}
Beyond the safety invariants stated in the main paper, the runtime also enforces the following.

\begin{invariant}[Effective-chain preservation]
A compensation is rejected if it would break the effective chain for an entity.
\end{invariant}

\begin{invariant}[Admitted contract creation]
Creating an \ACR{} is itself an admitted transition. The record must have bounded dependency scope, a discharge or expiry condition, a guard, a validator, and an idempotent creation key.
\end{invariant}

\begin{invariant}[Admitted readiness]
A watcher may propose that a live \ACR{} is ready or breached only from admitted
evidence. The lifecycle state changes only if admission commits that proposal, and a
continuation may resume only from the resulting committed ready or breached state.
\end{invariant}

\begin{invariant}[Scoped repair]
A proposal generated by an \ACR{} may touch only the dependency scope recorded by that \ACR{}, unless a wider scope is explicitly admitted as an escalation.
\end{invariant}

\begin{invariant}[Cross-episode traceability]
Rejected or compensated \ACR{} proposals with matching failure signatures remain queryable as historical evidence for future proposal ranking, guard strengthening, or external causal-learning modules.
\end{invariant}

\section{Benchmark Roster}
\label{s:roster}
The main paper's evaluation (\S6) summarizes the versioned benchmark suite compactly; the full
roster is as follows. The six safety benchmarks check, respectively, that generated proposals stay
non-authoritative until admitted (\textsc{AuthorityBench}); that concurrent proposals commit to a
serial-equivalent history (\textsc{SerialAdmissionBench}); that a repair cannot discharge its own
trigger by destroying evidence (\textsc{EvidenceRepairBench}); that contract wakeups resume
recovery without mutating domain truth (\textsc{ObligationBench}); that compensation and
supersession preserve effective-state projection (\textsc{CompensationProjectionBench}); and that
storage-level uniqueness, idempotency, and transactional projection reject duplicate or malformed
attempts without corrupting \StateView{} (\textsc{StorageSubstrateBench}).
\textsc{J1J4EndToEndBench} drives planning and recovery cases through the boundary,
\textsc{ProposerQualityBench} varies proposer quality while tracking usefulness and the
invalid-commit invariant, and \textsc{GuardrailComparatorBench} compares \ATP{} with a
workflow/saga guardrail stack (schema validation, finite-state checks, idempotency keys, timers,
local saga compensation, proposer self-checking) that lacks \ATP{}-specific admission over stale-world assumptions, effective state, evidence
preservation, obligation containment, dependency-closed compensation, and conflict scopes. This
14-case mechanism audit is the ninth bar in the aggregate scoreboard. A separate
\textsc{TemporalProductBaseline} runs one valid operation and three semantic hazards through an
actual Temporal SDK worker; it is reported separately because it is a product check, not a tenth
permissive mechanism benchmark.

\section{Aggregate Safety Scoreboard}
This is the full aggregate safety scoreboard summarized in the main-paper evaluation.

\begin{table}[t]
\centering
\caption{Aggregate safety scoreboard across the nine benchmarks: each permissive baseline commits the targeted violation, while \ATP{} commits zero and still admits valid operations.}
\label{tab:current-rq-results}
\footnotesize
\setlength{\tabcolsep}{4pt}
\begin{tabularx}{\linewidth}{l l >{\raggedright\arraybackslash}X >{\raggedright\arraybackslash}X}
\toprule
\rowhead
Class & Benchmark & Baseline violations & ATP violations \\
\midrule
Authority & AuthorityBench & 35 invalid generated commits & 0 invalid generated commits \\
Serial & SerialAdmissionBench & 48 invalid commits and capacity underflow & 0 invalid commits; serial-equivalent history \\
Evidence & EvidenceRepairBench & 20 evidence-destroying repairs & 0 evidence-destroying repairs \\
Obligation & ObligationBench & 16 unauthorized wakeup mutations & 0 unauthorized wakeup mutations \\
Projection & CompensationProjectionBench & 7 \StateView{} mismatches & 0 \StateView{} mismatches \\
Storage & StorageSubstrateBench & 64 invalid storage commits and 1 \StateView{} mismatch & 0 invalid commits and 0 \StateView{} mismatches \\
End-to-end & J1J4EndToEndBench & 3 invalid commits and 3 \StateView{} mismatches & 0 invalid commits, 0 mismatches, 4 completed cases \\
Proposer & ProposerQualityBench & 37 invalid direct commits & 0 invalid commits across six proposer classes \\
Mechanism & GuardrailComparator & 6 ATP-specific hazards missed by workflow/saga guardrails & 0 invalid commits; 4 valid commits admitted \\
\bottomrule
\end{tabularx}
\end{table}

\section{Per-Proposer-Class Detail}
This table gives the per-proposer-class detail behind the proposer-quality study of the main paper (the corresponding bar chart appears in the main paper).

\begin{table}[t]
\centering
\caption{RQ4 proposer-level result under ATP. Proposer quality changes acceptance
rate, utility, and first-admission delay; invalid commits remain zero for every
proposer class.}
\label{tab:rq3-proposer-results}
\setlength{\tabcolsep}{4.5pt}\small
\begin{tabular}{lrrrrrr}
\toprule
\rowhead
Proposer & Attempts & Admitted & Rejected & Accept. rate & Utility & First admission \\
\midrule
No intelligence & 10 & 1 & 9 & 0.10 & 1 & 8 \\
Random & 10 & 2 & 8 & 0.20 & 4 & 5 \\
Rule-based & 10 & 5 & 5 & 0.50 & 25 & 2 \\
LLM-like & 10 & 6 & 4 & 0.60 & 36 & 2 \\
Solver-like & 10 & 8 & 2 & 0.80 & 64 & 1 \\
Adversarial & 10 & 1 & 9 & 0.10 & 1 & 10 \\
\bottomrule
\end{tabular}
\end{table}

\section{Concurrency, Contention, and Distributed Execution}
\textbf{Concurrency and contention.} The serialization unit is the conflict scope, not the whole
log: proposals over disjoint scopes (distinct tenant, entity, or resource) commit concurrently, so
throughput scales with scope-level parallelism rather than a single global lock. Within one hot
scope, the PostgreSQL substrate orders contenders through row-level locking and a unique constraint
over (tenant, scope, idempotency key, recovery sequence position), so concurrent proposals either
converge to one canonical committed transition or reject the conflicting losers, exactly the boundary
of the Serial-Equivalent Generative Admission Theorem; MVCC snapshot isolation lets the \StateView{} projection read effective
state without blocking admission. A large disruption that wakes many obligations therefore degrades
to serialized admission \emph{on the affected scopes only}, trading latency for serial equivalence
rather than risking an invalid interleaving, and rejected contenders retry against the updated
\StateView{} under the memoryless repair model (the Reactive Repair section of the main paper). Cross-scope independence is what
keeps dense conflict-scope collisions from becoming a global bottleneck.

\textbf{Distributed execution.} The evaluation runs against a single durable
store, which is the linearization point for committed truth: each tenant and conflict scope has one
authoritative \CTL{}, so the authority boundary is well defined within a store even when proposers,
solvers, and orchestrators are distributed. External orchestration engines such as Temporal drive
execution but never hold commit authority (Appendix~\ref{app:eng-engine}); under network partitions
or worker-node failure they re-drive activities at least once, and the operation idempotency key
together with the outbox deduplication boundary make replays converge to a single committed
transition and a single outbox intent. If the provider also honors the correlation or idempotency
key, retries need not duplicate the external effect. What we do not yet evaluate
is multi-region or sharded operation, where committed truth would span several stores: maintaining
serial-equivalent admission and outbox idempotency across asynchronous cross-region updates would
require partitioning conflict scopes across stores or a shared consensus log, which, with
production-scale load testing, we leave to future work.

\section{Per-Property Safety Results}
\label{app:safety-tables}
This appendix gives the full per-class result tables for the six safety hazard classes
summarized in RQ1 (the main-paper evaluation) and visualized in
the safety-scoreboard figure of the main paper; each compares \Mnemo{} against the deliberately
permissive baseline(s) for that class.
Table~\ref{tab:safety-authority} reports \textsc{AuthorityBench}: a raw-append baseline and
a self-validation baseline each commit $35$ invalid generated transitions, while \ATP{}
commits zero and still admits the five valid proposals.
Table~\ref{tab:safety-serial} reports \textsc{SerialAdmissionBench}: under $80$ concurrent
proposals over a shared capacity object, the unserialized and weak-lock baselines commit
invalid proposals and underflow capacity, while \ATP{} commits only the $32$ valid
proposals through a serialized boundary and yields a serial-equivalent history.
Table~\ref{tab:safety-evidence} reports \textsc{EvidenceRepairBench}: the naive-repair and
no-rule workflow baselines each commit $20$ evidence-destroying repairs, while \ATP{}
commits zero and still admits the eight valid repairs.
Table~\ref{tab:safety-obligation} reports \textsc{ObligationBench}: the trigger-direct and
timer-direct baselines commit $12$ and $16$ unauthorized wakeup mutations, while \ATP{}
routes every wakeup through ordinary admission and commits zero.
Table~\ref{tab:safety-projection} reports \textsc{CompensationProjectionBench}: the
projection baselines admit five invalid compensations and produce seven \StateView{}
mismatches, while \ATP{} produces zero and still commits the two valid compensations.
Table~\ref{tab:safety-storage} reports \textsc{StorageSubstrateBench}: the unconstrained
log commits $64$ invalid storage attempts and one \StateView{} mismatch, while the \ATP{}
storage path rejects all $64$ with zero mismatches.

\begin{table}[ht]
\centering
\caption{Authority separation.  Unsafe baselines commit invalid generated
proposals; ATP rejects all invalid proposals while still admitting valid ones.}
\label{tab:safety-authority}
\setlength{\tabcolsep}{4pt}
\begin{tabular}{lrrr}
\toprule
\rowhead
System & \makecell{Invalid\\commits} & \makecell{Rejected\\invalid} & \makecell{Valid\\commits} \\
\midrule
Raw append & 35 & 0 & 5 \\
Self-validation & 35 & 0 & 5 \\
ATP / \Mnemo{} & 0 & 35 & 5 \\
\bottomrule
\end{tabular}
\end{table}

\begin{table}[ht]
\centering
\caption{Serial-equivalent admission under concurrency. Unsafe baselines commit invalid
concurrent proposals and produce non-serial-equivalent state; ATP admits only
valid proposals through a serialized transaction boundary. Column SEA denotes
serial equivalence.}
\label{tab:safety-serial}
\setlength{\tabcolsep}{3.5pt}\footnotesize
\begin{tabular}{@{}lrrrrc@{}}
\toprule
\rowhead
System & Committed & Rejected & Invalid & Underflow & SEA \\
\midrule
Unserial.\ writes & 80 & 0 & 48 & 1 & no \\
Weak lock & 68 & 12 & 36 & 1 & no \\
ATP / \Mnemo{} & 32 & 48 & 0 & 0 & yes \\
\bottomrule
\end{tabular}
\end{table}

\begin{table}[t]
\centering
\caption{Evidence-preserving repair. Unsafe baselines admit repairs that
hide trigger evidence while leaving the underlying failure unresolved; ATP
rejects all such repairs and still admits valid repairs.}
\label{tab:safety-evidence}
\setlength{\tabcolsep}{4pt}
\begin{tabular}{lrrr}
\toprule
\rowhead
System & \makecell{Evidence-\\destroying} & \makecell{Rejected\\evid.-dest.} & \makecell{Valid\\repairs} \\
\midrule
Naive repair & 20 & 0 & 4 \\
Workflow w/o rule & 20 & 0 & 8 \\
ATP / \Mnemo{} & 0 & 20 & 8 \\
\bottomrule
\end{tabular}
\end{table}

\begin{table}[t]
\centering
\caption{Obligation containment. Unsafe baselines allow wakeups to mutate
committed domain state without ordinary admission; ATP allows wakeups to emit
proposal packages but rejects unauthorized mutations before commit.}
\label{tab:safety-obligation}
\setlength{\tabcolsep}{4pt}
\begin{tabular}{lrrrr}
\toprule
\rowhead
System & \makecell{Unauthorized\\mutations} & \makecell{Proposal\\packages} & Admitted & Rejected \\
\midrule
Trigger direct & 12 & 0 & 0 & 0 \\
Timer direct write & 16 & 0 & 0 & 0 \\
ATP / \Mnemo{} & 0 & 24 & 4 & 20 \\
\bottomrule
\end{tabular}
\end{table}

\begin{table}[t]
\centering
\caption{Effective-state and compensation safety.  Unsafe projection
baselines admit invalid compensations and project ineffective history as current
truth; ATP rejects invalid compensations and preserves \StateView{} as the
projection of effective committed records only.}
\label{tab:safety-projection}
\small
\begin{tabular}{lrrrrr}
\toprule
\rowhead
System & \makecell{Invalid\\admitted} & \makecell{Orphaned\\deps.} & \makecell{Broken\\chains} & \makecell{\StateView{}\\mismatches} & Valid \\
\midrule
Unsafe projection baselines & 5 & 2 & 5 & 7 & 2 \\
ATP / \Mnemo{} & 0 & 0 & 0 & 0 & 2 \\
\bottomrule
\end{tabular}
\end{table}

\begin{table}[t]
\centering
\caption{Storage-substrate correctness.  The unconstrained log baseline
commits duplicate or malformed storage attempts; the ATP storage path rejects
invalid attempts and preserves effective-state projection. The PostgreSQL-backed store is implemented and exercised by environment-gated live conformance tests; the default artifact path is PostgreSQL-free.}
\label{tab:safety-storage}
\setlength{\tabcolsep}{3.5pt}\footnotesize
\begin{tabular}{llrrrr}
\toprule
\rowhead
System & Subs. & Committed & Rejected & Invalid & Mismatch \\
\midrule
Unconstr.\ log & mem & 128 & 0 & 64 & 1 \\
ATP storage & PostgreSQL & 64 & 64 & 0 & 0 \\
\bottomrule
\end{tabular}
\end{table}

\section{Live-Proposer Pilots (RQ3): Detail and Representative Cases}
\label{s:rq6-detail}
This section records the full detail behind the main paper's RQ3: the pilot configurations, the
complete scorer outputs, the staged pipeline, and the representative cases referenced there. Both
pilots are bounded: one task family (job-shop), four proposer packs (Claude, GPT, DeepSeek expert,
DeepSeek instant), manual collection, and no claim of learning efficacy or family generality.

\subsection{Static plan-entry pilot}
The static pilot uses one Tier~6 E7 sequence, \texttt{T6-7e17ef0cc5f3}, under the full C+R+T stack.
Four response packs contribute ten episodes each, giving forty handoff-derived cases; all forty
validate with zero schema failures. Scorer outputs: safety gate passed (zero invalid commits, zero
evidence-destroying repairs, zero orphaned dependents); mean horizon reward 0.90125; grounded
admission rate 0.725. The staged pipeline is: collect and parse the packs; admit or reject each
proposal; attach the kernel-admission trace; replay through the runtime evaluator; bridge admission
labels to scorer actions; export the handoff bundle; import and score in REALM-Bench. Every stage
preserves the authority boundary: live models propose, \Mnemo{} admits or rejects, the external
scorer reads boundary traces only.

\subsection{Dynamic mid-execution pilot}
The dynamic pilot poses ten mid-execution disruption episodes over the job-shop E7 setting; each of
the four packs answers all ten (40 responses of type \texttt{repair}, \texttt{reject}, or
\texttt{observe}). Deterministic admission guards require committed operations to remain unchanged,
committed evidence to be preserved, rollback scope to remain none or local, and repair radius to
remain bounded. Scorer outputs: 24 admitted, 16 rejected (4 explicit safety rejections); safety
gate passed; repeated-failure rate 0.0; mean horizon reward 0.925; grounded admission rate 0.6;
observed mean time-to-correction 0.725 (40 observed, 0 censored).

\subsection{Representative cases}
\emph{Admitted local repair.} Claude, episode~1, proposed moving the uncommitted operation
\texttt{J4-O2} away from a failed \texttt{M2} machine while preserving committed operations
\texttt{J2-O1} and \texttt{J3-O1}; the affected step was local, rollback scope none, evidence
preserved, and the proposal passed admission.

\emph{Safe rejection.} DeepSeek expert, episode~3, proposed a repair touching two affected steps,
\texttt{J2-O2} and \texttt{J3-O2}, when the repair-radius limit was one; \Mnemo{} rejected it with
the structured reason \texttt{repair\_radius\_exceeded:2>1}. A bad proposal plus a correct
pre-commit rejection is a safety success, not a system failure.

\emph{Refused global rollback.} Claude, episode~9, itself rejected an unsafe global rollback
request, recognizing that rollback would violate committed evidence for \texttt{J2-O1} and
\texttt{J3-O1}. The replay records this conservatively as a model-side rejection rather than a
guard-derived safety rejection, because it was requested by the proposer rather than triggered by
an admission guard; it illustrates that structured rejection reasons can serve as feedback for
iterative replanning, which is out of scope for this paper.

\section{RQ2 Two-Tier Workflow-Engine Comparison}
\label{s:rq2-workflow}

RQ2 deliberately separates a broad mechanism audit from a narrow executable product check.
\textsc{GuardrailComparatorBench} contains 14 cases: four valid operations, four classical
guardrail hazards, and six ATP-specific hazards. The workflow/saga mechanism model rejects the
four classical hazards but commits all six semantic hazards; ATP admits the four valid operations
and rejects all ten invalid ones. This broad audit supplies the ``Mechanism'' bar in the main
paper's Figure~4.

\textsc{TemporalProductBaseline} then executes four matched cases using Temporal SDK 1.32.0,
its test server, an actual worker and activities, and SQLite persisted domain state. Both arms
apply the same schema, finite-state, idempotency-key, and proposer-self-check guardrails. The ATP
arm additionally validates effective state, retained evidence, and compensation closure before
the activity writes domain state.

\begin{table}[t]
\centering
\caption{Executable Temporal product check. Every workflow has an 11-event Temporal history.}
\label{tab:temporal-product}
\setlength{\tabcolsep}{5pt}\small
\begin{tabular}{lrrrrr}
\toprule
\rowhead
System & Cases & Commit & Reject & Invalid & Valid \\
\midrule
Temporal guardrails & 4 & 4 & 0 & 3 & 1 \\
Temporal $+$ ATP & 4 & 1 & 3 & 0 & 1 \\
\bottomrule
\end{tabular}
\end{table}

The three hazards are a stale-world plan, an evidence-destroying unresolved repair, and a
compensation that orphans an effective dependent. The result does not identify a Temporal defect:
Temporal provides durable orchestration, and application code remains responsible for semantic
admission. ATP supplies one explicit contract for that responsibility.

\subsection{Documented LangGraph and Cadence Capability Boundary}

Official LangGraph documentation describes a low-level orchestration runtime with persistence,
durable execution, stores, configurable conditional interrupts, and human approve/edit/reject
decisions. Official Cadence documentation describes durable workflow history, retries, timers,
signals, activities, worker scaling, and saga compensation whose logic is written by the
application. Table~\ref{tab:workflow-capabilities} reports documented built-ins, not what a
sufficiently determined application developer could add.

\begin{table}[t]
\centering
\caption{Capability-boundary audit from official documentation. ``Application'' means that the
framework can host the logic but does not document the ATP property as a built-in guarantee.}
\label{tab:workflow-capabilities}
\setlength{\tabcolsep}{4pt}\footnotesize
\begin{tabularx}{\linewidth}{>{\raggedright\arraybackslash}X l l}
\toprule
\rowhead
Capability & LangGraph & Cadence \\
\midrule
Durable state/history and resume & built in & built in \\
Retries, timers, event-driven continuation & graph/application & built in \\
Human approval or interruption & built in/configurable & application/signals \\
Saga compensation & application & application pattern \\
Schema, FSM, idempotency checks & application & application \\
Effective-state admission witness & not documented & not documented \\
Evidence-preserving repair & not documented & not documented \\
Dependency-closed compensation & not documented & not documented \\
Non-authoritative obligation wakeup & not documented & not documented \\
Generative conflict-scope serialization & not documented & not documented \\
\bottomrule
\end{tabularx}
\end{table}

The source basis is the official LangGraph overview and human-in-the-loop documentation
\cite{langgraph_overview_2026,langgraph_hitl_2026}, and the official Cadence workflow-engine and
orchestration documentation \cite{cadence_engine_2026,cadence_orchestration_2026}. Because both
systems are programmable, ``not documented'' must not be read as ``impossible to implement.'' It
means only that the safety property is not supplied as a framework-level contract that an
application receives by default.

The product check is reproduced by:
\begin{verbatim}
python -m pytest -q -m temporal \
  tests/benchmarks/test_temporal_product_baseline.py
\end{verbatim}
The broader mechanism audit remains reproduced by
\path{tests/experiments/test_rq9_state_of_practice_comparator.py}.

\section{Coverage Audit Against Production TP}
\label{app:tp-coverage}
This appendix extends the RQ2 comparator (the RQ2 comparator table of the main paper) with a third layer: a
production transaction-processing (TP) substrate such as PostgreSQL. ATP is not a replacement for
production TP. PostgreSQL-style systems provide the durable substrate, with atomicity, isolation,
durability, uniqueness, foreign keys, checks, triggers, and serializable execution. ATP is complementary:
it governs whether a generated proposal should receive transaction authority before it becomes an ordinary
transaction. The distinction matters because several agentic hazards are not storage-level violations
unless application developers manually encode ATP-specific semantics.

Table~\ref{tab:tp-coverage} reports the benchmark's coverage table. It is a deterministic
coverage audit, not a throughput benchmark: each entry is a by-design assessment of which layer naturally owns the corresponding check.

\begin{table}[t]
\centering
\caption{Coverage audit: a production TP substrate versus workflow/saga guardrails versus ATP. Entries are
by-design assessments of which layer owns each check, not measured throughput.}
\label{tab:tp-coverage}
\setlength{\tabcolsep}{4pt}\footnotesize
\begin{tabularx}{\linewidth}{>{\raggedright\arraybackslash}X l l >{\raggedright\arraybackslash}p{2.9cm}}
\toprule
\rowhead
Hazard & PostgreSQL-style TP & Workflow/saga & ATP/\Mnemo{} \\
\midrule
Primary-key / unique duplicate & caught & caught & caught \\
Missing foreign-key dependency & caught & caught & caught \\
Finite-state invalid transition & app-only & caught & caught \\
Stale-world proposal & missed & missed & caught \\
Orphaning compensation over effective state & app-only & partial & caught \\
Evidence-destroying repair & missed & missed & caught \\
\ACR{} direct domain mutation & missed & missed & caught \\
Generative conflict-scope collision & app-only & partial & caught \\
Duplicate side-effect intent & partial & partial & caught (staged, then rejected) \\
\bottomrule
\end{tabularx}
\end{table}

The audit identifies a layer boundary rather than a weakness in PostgreSQL or other TP systems. Production
TP systems are strong commit substrates, but they do not by themselves define which generated proposal
should be granted transaction authority. ATP adds \StateView{}-based admission, retained-evidence checks,
dependency-closed compensation, non-authoritative \ACR{} wakeups, and generative conflict scopes before a
proposal becomes a transaction; once a proposal is admitted, ATP relies on a production TP substrate for
atomic commit, durability, uniqueness, and isolation.

We implement this audit as \textsc{ProductionTPComparatorBench}, a
deterministic coverage benchmark rather than a throughput benchmark. The
benchmark runs the same hazard set through three enforcement models:
a PostgreSQL-style TP substrate, a workflow/saga guardrail layer, and
ATP/\Mnemo{}. Its purpose is to identify which layer owns each
enforcement responsibility, not to claim that production TP systems are
weak or that PostgreSQL cannot encode application-specific triggers.

The benchmark covers nine hazards. ATP/\Mnemo{} catches all nine.
The PostgreSQL-style TP substrate catches the two storage-level hazards
natively, marks three as application-encoded, partially catches one, and
misses three ATP-specific proposal-authority hazards. The workflow/saga
guardrail layer catches three classical guardrail hazards, partially
catches three, and misses three ATP-specific hazards.

The benchmark is reproduced by:
\begin{verbatim}
python benchmarks/realm/production_tp_comparator.py
python -m pytest -q tests/benchmarks/test_production_tp_comparator.py
\end{verbatim}
It emits \texttt{summary.json}, \texttt{summary.csv}, and
\texttt{report.md} under:
\begin{verbatim}
benchmarks/realm/reports/
production_tp_comparator/
\end{verbatim}
The exact VLDB artifact snapshot is:
\begin{verbatim}
https://github.com/eyuchang/Mnemosyne/tree/arxiv-atp-rq1-rq9b-r8-v2
\end{verbatim}

\section{RQ5 Infrastructure-Cost Audit: Full Method and Results}
\label{s:rq4-cost}
This section records the full method and result table behind the main paper's RQ5 summary.

RQ2 establishes that ATP rejects the ATP-specific hazards in the mechanism-level
workflow/saga comparator; RQ5 audits the cost visible in the implemented ATP
infrastructure paths. We do not time the RQ2 oracle comparator as though it were
the ATP runtime, and we do not claim a production load test. Instead, we run a
repository infrastructure workload covering kernel admission, runtime admission,
commit-boundary enforcement, effective-state projection, validation,
compensation, active recovery, and Temporal activity-boundary validation. Each
repeat executes 20 ATP infrastructure paths, and we measure 30 repeats at worker
counts $1$, $4$, $8$, and $16$. The workflow/saga semantic comparator is still executed
to preserve the RQ2 baseline, but its single semantic test is not treated as a
throughput competitor for the 20-path ATP infrastructure workload.

\begin{table}[ht]
\centering
\caption{RQ5 artifact-level infrastructure-cost audit for ATP's implemented safety paths. Each run
executes 20 real repository paths covering admission, commit-boundary enforcement,
effective-state projection, validation, compensation, active recovery, and
Temporal activity-boundary validation. Latencies are $p_{50}/p_{95}$ in
milliseconds; throughput is infrastructure-test units per minute. All settings use
30 repeats.}
\label{tab:rq5-cost-results}
\setlength{\tabcolsep}{3.8pt}\scriptsize
\resizebox{\textwidth}{!}{%
\begin{tabular}{lrrrrrrrr}
\toprule
\rowhead
Workers & Pass & \makecell{E2E\\ms} & \makecell{Admission\\ms} &
\makecell{Commit\\ms} & \makecell{Projection\\ms} & \makecell{Validation\\ms} &
\makecell{PV\\overhead} & \makecell{Throughput\\units/min} \\
\midrule
1 & 30/30 & 5328.2 / 5449.2 & 0.041 / 17.438 & 2.597 / 8.652 & 0.023 / 1.754 & 0.001 / 0.058 & 11.82\% & 224.53 \\
4 & 30/30 & 6874.8 / 7323.7 & 0.059 / 22.877 & 3.079 / 9.360 & 0.025 / 2.295 & 0.001 / 0.070 & 10.77\% & 661.28 \\
8 & 30/30 & 12500.8 / 12646.6 & 0.102 / 37.686 & 5.154 / 13.313 & 0.030 / 3.667 & 0.001 / 0.097 & 9.62\% & 768.35 \\
16 & 30/30 & 62829.8 / 76012.6 & 0.006 / 487.950 & 15.311 / 76.245 & 0.064 / 17.305 & 0.008 / 0.397 & 4.59\% & 303.70 \\
\bottomrule
\end{tabular}
}
\end{table}

Table~\ref{tab:rq5-cost-results} reports the real-infrastructure result. All ATP
infrastructure runs pass at all four worker counts. Throughput rises from $224.53$ units/min at
one worker to $661.28$ at four and $768.35$ at eight, but drops to $303.70$ at 16 while end-to-end
$p_{50}$ rises to $62.8$ seconds. The pytest-process audit therefore saturates after eight workers
on this host. Its instrumentation fractions must not be interpreted as end-user latency overhead:
they classify profiled function time inside an end-to-end pytest run. The value of this audit is
path coverage and a visible saturation point, not a production throughput claim.

\subsection{Matched Data-Size Scaling}

To isolate admission cost from pytest process startup, a second benchmark compares two valid-operation
paths over the same SQLite WAL CTL and \StateView{} substrate. The commit-only arm constructs trusted
CTL records and performs atomic append plus projection. The ATP arm first validates generated
transition candidates against the FSM and current \StateView{}, constructs admitted records, and
then performs the identical commit plus projection. Independent entity/workflow operations are
grouped into fixed 100-operation transactions. Each size is repeated five times with arm order
alternated between repeats.

\begin{table}[ht]
\centering
\caption{RQ5 matched data-size sweep. Latencies are amortized per operation; throughput and
latency entries are medians over five repeats. Every run has the expected CTL row count, zero
sampled \StateView{} mismatches, and zero validation failures.}
\label{tab:rq5-data-scale}
\setlength{\tabcolsep}{4pt}\small
\begin{tabular}{lrrrrr}
\toprule
\rowhead
Mode & Ops & Elapsed ms & Ops/s & $p_{50}/p_{95}$ ms & ATP cost \\
\midrule
Commit only & $1{,}000$ & $314.6$ & $3179.0$ & $0.298/0.381$ & -- \\
ATP & $1{,}000$ & $336.8$ & $2969.4$ & $0.322/0.427$ & $6.59\%$ \\
Commit only & $5{,}000$ & $4634.6$ & $1078.8$ & $0.879/1.701$ & -- \\
ATP & $5{,}000$ & $4890.9$ & $1022.3$ & $0.953/1.838$ & $5.24\%$ \\
Commit only & $10{,}000$ & $16863.5$ & $593.0$ & $1.640/3.156$ & -- \\
ATP & $10{,}000$ & $18026.6$ & $554.7$ & $1.731/3.596$ & $6.45\%$ \\
\bottomrule
\end{tabular}
\end{table}

The incremental ATP cost stays between $5.2\%$ and $6.6\%$ across the tested sizes. Both arms slow
as the local CTL and projection tables grow, so this experiment does not establish production
scalability. It shows more narrowly that semantic validation is not the dominant source of the
observed local slowdown. Together, the worker and data-size audits replace the earlier suggestion
of unconstrained scaling with a measured claim: the safety paths remain correct through the tested
loads, admission adds modest incremental cost, and the present local implementation has clear
saturation points.

The data-size sweep is reproduced by:
\begin{verbatim}
python -m benchmarks.realm.rq5_data_scale
python -m pytest -q tests/benchmarks/test_rq5_data_scale.py
\end{verbatim}

Together, RQ2 and RQ5 report the safety distinction, an executable Temporal check, and local
artifact costs; they do not establish production-scale throughput or a deployed-service comparison.

\section{RQ6 Controlled Recovery-During-Recovery Stress}
\label{s:rq6-recursive}

The main paper's recursive-recovery proposition reduces a mid-repair disruption to a fresh
single-shot repair over the latest committed restore point $S_k$ and the enlarged pending set
$D_k$. The executable stress harness attacks the proposition's implementation premises at four
checkpoints: before readiness admission, after readiness but before remedy construction, during
compensation, and immediately before repair commit. Each checkpoint is exercised at failure depths
one, two, and three, for 12 scenarios with per-repair round bound $K=3$.

\begin{table}[t]
\centering
\caption{RQ6 recursive-recovery stress results.}
\label{tab:rq6-recursive}
\setlength{\tabcolsep}{5pt}\small
\begin{tabular}{lr}
\toprule
\rowhead
Measure & Result \\
\midrule
Injection checkpoints & 4 \\
Failure depths per checkpoint & 1--3 \\
Total scenarios & 12 \\
Stale in-flight candidates rejected & 12 \\
Invalid commits & 0 \\
Lost pending observations & 0 \\
Duplicate commits & 0 \\
Realized bound violations & 0 \\
Maximum realized rounds / largest $FK$ bound & $3/9$ \\
\bottomrule
\end{tabular}
\end{table}

Every new disruption is appended to $D_k$ before another repair can earn admission. Consequently,
the stale in-flight candidate is rejected, and the final admitted repair covers all pending
observations. This is deterministic controlled stress, not a distributed worker-crash, network
partition, or production-runtime experiment.

The stress matrix is reproduced by:
\begin{verbatim}
python -m pytest -q tests/benchmarks/test_recursive_recovery_stress.py
\end{verbatim}

\section{R8 Deployment Smoke-Load Variability}
\label{app:r8}
R8 implements a deployable \Mnemo{} service boundary with health, proposal-submission, \StateView{}, and metrics endpoints. The purpose of the R8 smoke-load audit is not to claim production-scale throughput, but to check that the deployment service preserves ATP's authority-separation rule under concurrent HTTP proposal submission: clients may submit proposals, but invalid proposals and explicit bypass attempts must not become committed truth.

Table~\ref{tab:r8-smoke-variability} reports two consecutive local runs of the same deployment load script. Each worker setting submits 200 HTTP proposal requests. The workload contains 120 valid proposals and 80 invalid or bypass proposals. In both runs, every worker setting admits exactly 120 proposals, rejects exactly 80 proposals, and produces zero invalid commits. Thus the safety result is stable across repeated local runs.

The throughput measurements are intentionally treated as diagnostic rather than production evidence. The 16-worker setting shows substantial run-to-run throughput variability despite similar $p_{50}/p_{95}$ request latency, indicating sensitivity to local client scheduling, socket backlog, and short-run wall-clock effects. We therefore use R8 only as deployment-boundary evidence: concurrent HTTP submission preserves the ATP admission boundary, while production load testing over durable PostgreSQL storage, Kubernetes workers, networked workflow services, and pool saturation remains future work.

\begin{table}[t]
\centering
\small
\setlength{\tabcolsep}{4pt}
\caption{R8 local deployment smoke-load variability across two consecutive runs. Each setting submits 200 HTTP proposal requests to the R8 service. The stable result is safety: invalid commits remain zero in every run and worker setting. Throughput is local diagnostic evidence, not a production load-test claim.}
\label{tab:r8-smoke-variability}
\begin{tabular}{lrrrrr}
\toprule
\rowhead
Run & Workers & \makecell{Accepted /\\Rejected} & \makecell{Invalid\\commits} & \makecell{Latency\\$p_{50}/p_{95}$ ms} & \makecell{Throughput\\/min} \\
\midrule
1 & 1 & 120 / 80 & 0 & 0.281 / 0.381 & 119{,}717.96 \\
1 & 4 & 120 / 80 & 0 & 0.871 / 1.515 & 243{,}307.15 \\
1 & 8 & 120 / 80 & 0 & 1.744 / 2.892 & 245{,}481.60 \\
1 & 16 & 120 / 80 & 0 & 3.398 / 6.226 & 245{,}801.74 \\
\midrule
2 & 1 & 120 / 80 & 0 & 0.250 / 0.380 & 173{,}746.85 \\
2 & 4 & 120 / 80 & 0 & 0.822 / 1.307 & 271{,}099.39 \\
2 & 8 & 120 / 80 & 0 & 1.743 / 3.136 & 257{,}312.64 \\
2 & 16 & 120 / 80 & 0 & 3.047 / 5.396 & 11{,}784.96 \\
\bottomrule
\end{tabular}
\end{table}

\section{Mnemosyne Architecture and Runtime Data Structures}
\label{app:eng-arch}
\begin{figure}[t]
\centering
\resizebox{0.6\linewidth}{!}{%
\begin{tikzpicture}[every node/.style={font=\small}]
\node[gate, minimum width=32mm] (api) at (3.4,1.5) {Store protocol};
\node[store, minimum width=23mm] (pg) at (0,0) {PostgreSQL\\\scriptsize substrate};
\node[store, minimum width=23mm] (pool) at (3.4,0) {Pooled PG\\\scriptsize concurrency};
\node[view, minimum width=23mm] (sq) at (6.8,0) {SQLite\\\scriptsize fallback};
\node[draw=atpslate!60, fill=atpslate!6, dashed, rounded corners=2pt, align=center,
      inner sep=4pt, minimum width=46mm] (k8s) at (3.4,-1.6) {Scalability\,/\,Kubernetes port};
\draw[commitflow] (api)--(sq);
\draw[commitflow] (api)--(pg);
\draw[commitflow] (api)--(pool);
\draw[orch] (sq)--(k8s);
\draw[orch] (pg)--(k8s);
\draw[orch] (pool)--(k8s);
\end{tikzpicture}%
}
\caption{Storage substrate boundary. PostgreSQL (and pooled PostgreSQL for concurrency) is the substrate for all reported experiments; SQLite remains a zero-setup fallback for portability. The backing store is an implementation choice behind the store protocol: it affects neither ATP correctness semantics nor any reported number. The scalability and Kubernetes deployment phase changes deployment scale, not correctness.}
\Description{A substrate diagram shows a store protocol connected to PostgreSQL as the experiment substrate, pooled PostgreSQL for concurrency, SQLite as a zero-setup portability fallback, and a scalability or Kubernetes port that changes deployment scale without changing ATP correctness semantics.}
\label{fig:storage}
\end{figure}
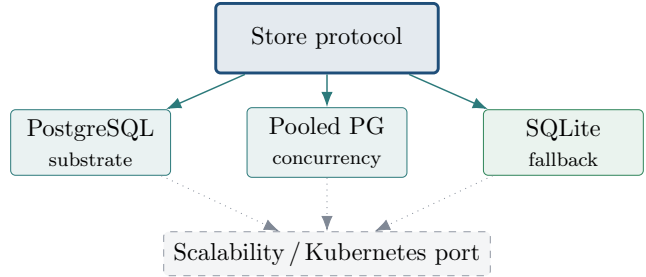

This appendix specifies the runtime data structures that realize the architecture
of the Architecture section of the main paper at the level of record fields. The intent is to make
the authority boundary inspectable: every object below is either committed truth, a
projection of committed truth, a durable obligation, or an external boundary, and no
proposer-side object can become truth except by passing admission.

\paragraph{Object taxonomy.}
\Mnemo{} keeps four kinds of object apart, and logs only the last. An \emph{event} is
an immutable observed fact (for example, a provider confirmation or a disruption
notice). A \emph{candidate} is a proposed finite-state transition that an admitted
commit may or may not produce. A \emph{constraint evaluation} is a predicate that the
admission validator checks. A \emph{committed transition} is the durable record of an
admitted transition: a \CTL{} record. Events, constraints, and finite-state structure
live outside the \CTL{}; only an admitted, committed transition is a \CTL{} record. A
vetoed or rejected proposal leaves a queryable rejection or escalation entry, never a
committed transition.

\paragraph{The committed-transition log record.}
A \CTL{} record is the unit of committed truth. Its identity and ordering fields are
fixed; its recovery metadata and domain evidence are carried in two typed
side-channels, the metadata map $M$ and the typed extension $X$, so that recovery and
provenance fields can evolve without changing a transition's identity.
Table~\ref{tab:eng-ctl-schema} lists the fields.

\begin{table}[t]
\centering
\caption{\CTL{} record fields. The identity and ordering fields fix committed truth;
$M$ and $X$ are typed side-channels for recovery metadata and domain or causal
evidence. The record is logically append-only: later transitions may compensate or
supersede it, but never erase it.}
\label{tab:eng-ctl-schema}
\small
\begin{tabular}{p{0.16\linewidth}p{0.08\linewidth}p{0.65\linewidth}}
\toprule
\rowhead
Field & Symbol & Role \\
\midrule
\texttt{rid} & $r$ & Unique record identifier (primary key); the idempotency key when client-supplied. \\
\texttt{op\_id} & & Optional client idempotency key, used when \texttt{rid} is server-assigned. \\
\texttt{tx\_group\_id} & & Saga group: the records that compensate together (indexed, not unique). \\
\texttt{workflow\_id} & & Session or instance identity (for example, a ride or trip). \\
\texttt{binding\_id} & & Role-binding identity (an abstract role bound to entities). \\
\texttt{eid} & & Individual entity whose state changed. \\
\texttt{fsm} & & Finite-state machine to which the edge belongs. \\
\texttt{version} & $v$ & Monotonic version per \texttt{(eid, fsm)}; one ordered history per entity. \\
\texttt{state\_before} & $s^{-}$ & Pre-state; the reversal target for compensation. \\
\texttt{state\_after} & $s^{+}$ & Post-state of the committed transition. \\
\texttt{action\_type} & $a$ & Named action (for example, pickup, reroute, compensate); disambiguates self-loops. \\
\texttt{triggers} & $U$ & Proximate causes (event identifiers in the recovery-event log). \\
\texttt{dependencies} & $D$ & Required preconditions: \texttt{rid}s that must be committed and still effective. \\
\texttt{metadata} & $M$ & Verdict, compensation policy, restore flag, status, supersession or compensation links, provenance. \\
\texttt{extension} & $X$ & Schema-validated domain and causal-audit evidence plus temporal-spatial attributes. \\
\texttt{schema\_id} / \newline \texttt{schema\_version} & & Schema identity for safe replay across schema evolution. \\
\texttt{timestamp} & $t$ & Commit time; \texttt{log\_position} (below) provides ordering. \\
\bottomrule
\end{tabular}
\end{table}

The storage layer enforces two uniqueness constraints that carry the safety load:
\texttt{rid} is the primary key, so a retried commit cannot double-apply, and
\texttt{(eid, fsm, version)} is unique, so each entity has exactly one ordered
history. The saga group \texttt{tx\_group\_id} is indexed but deliberately
\emph{not} unique, so one saga spans many records. The metadata map $M$ carries the
recovery and audit fields, including \texttt{status} (one of \texttt{active},
\texttt{compensated}, or \texttt{superseded}), \texttt{restore} (whether the record
is a clean restore point), and the \texttt{compensates} or \texttt{supersedes} links
naming the \texttt{rid}s a record reverses or replaces. The typed extension $X$
carries domain attributes (location, ETA, deadline, route, assignment) and
causal-audit evidence; $X$ is schema-validated rather than free-form.

\paragraph{Identity and log position.}
\Mnemo{} keeps three identities distinct rather than collapsing them into one: the
individual entity \texttt{eid}, the role-binding \texttt{binding\_id}, and the session
or instance \texttt{workflow\_id}. Conflating them breaks repeated participation (one
entity acting twice), multi-party instances (one session with several participants),
and reassignment (a resource re-bound to a new role). Independently of the wall-clock
\texttt{timestamp}, the store assigns a monotonic \texttt{log\_position} that provides
total append ordering for replay and recovery; replay can proceed by per-entity
\texttt{version} or by global \texttt{log\_position}. The recovery-event log records
the fine-grained events of repair and obligation execution (packages created,
idempotency keys observed, sequence positions assigned, conflicts detected, outcomes
accepted or rejected). In the PostgreSQL substrate, tenant-scoped uniqueness over
event identifiers, idempotency keys, and recovery sequence positions provides the
concrete concurrent-idempotency boundary.

\paragraph{StateView: the effective-state projection.}
\StateView{} exposes current operational truth. It is reconstructed from the \CTL{}
by replaying only \emph{effective} records: records that have not been compensated,
superseded, or invalidated and whose dependency chains remain effective. Current
state is derived, never stored. Beyond the bare finite-state value, a \StateView{}
for an \texttt{(eid, fsm)} pair folds the $X$ attribute deltas forward (location,
ETA, deadline, route, assignment) and carries the set of effective records that
support it, so that admission and the reactive loop read where each entity currently
is, not a static plan. The projection is materialized per \texttt{(eid, fsm)} and is
a cache: replay from the \CTL{} remains the source of truth. A cross-entity
correction must re-project \emph{every} affected entity, because refreshing only the
issuing entity could leave a now-ineffective record lingering as current truth
elsewhere.

\paragraph{Active contract records.}
An \ACR{} is a durable workflow contract stored in the same committed substrate as
ordinary transitions. It records a future obligation or assumption, not a database
commit. For that reason the paper uses \emph{Active Contract Record}. The frozen
artifact retains \texttt{ActiveCommitment}, \texttt{active\_commitment}, and
\texttt{commitment\_fired} as backward-compatible code identifiers; they denote the
same ACR mechanism and do not grant commit authority.

Creating an \ACR{} is itself an admitted transition. A watcher that detects a trigger
constructs a readiness or breach lifecycle proposal grounded in admitted evidence.
The artifact operation \texttt{fire\_active\_commitment} implements this crossing: it
constructs a commitment-FSM candidate, validates it, and commits only the admitted
lifecycle record. The legacy status \textsf{fired} is therefore the artifact's name
for an admitted ready-or-breached state, not a direct watcher write. Only from that
committed state may the runtime resume the continuation to construct a remedy
proposal, which again passes ordinary admission before any domain transition can
commit.

The portable schema has three layers. A universal instance core identifies the
contract, clause, evidence, bounded scope, trigger, guard, continuation, validator,
and lifecycle. A reusable versioned profile can supply recurring remedy, budget,
expiry, idempotency, retry, and escalation policies. Sparse extensions add horizon,
standing, burden, external-effect correlation, or cross-episode fields only when a
profile or domain requires them. The tagged artifact implements the flat core in
Table~\ref{tab:eng-acr-schema}. A profile-enabled deployment would resolve its
profile and overrides into an effective record before admission. That normalization
is a forward-compatible schema refactoring specified here; profile resolution is not
implemented or evaluated in the tagged artifact.

\begin{table}[t]
\centering
\caption{Implemented flat-core \ACR{} fields in the tagged artifact. A future
profile-enabled input normalizes to an effective record before admission; it cannot
remove these safety-relevant requirements.}
\label{tab:eng-acr-schema}
\small
\begin{tabular}{p{0.30\linewidth}p{0.58\linewidth}}
\toprule
\rowhead
Field & Role \\
\midrule
\texttt{commitment\_id}, \texttt{commitment\_type} & Stable artifact identity and typed contract class. \\
\texttt{description} & The recorded clause, obligation, or assumption. \\
\texttt{creating\_record\_id}, \texttt{creating\_workflow\_id} & Admitted provenance that created the contract. \\
\texttt{dependency\_scope} & Bounded scope the resumed proposal may inspect or touch. \\
\texttt{trigger} & Condition from which a watcher may construct a readiness proposal. \\
\texttt{continuation\_ref}, \texttt{guard\_ref}, \texttt{validator\_ref} & Versioned re-entry point and the guard and validator controlling it. \\
\texttt{compensation\_ref} & Permitted compensation handler, if the selected profile allows one. \\
\texttt{priority} & Scheduling input; it never overrides admission. \\
\texttt{expiry} & Expiry or discharge condition. \\
\texttt{failure\_signature}, \texttt{cross\_episode\_key} & Optional recurrence and cross-episode traceability fields. \\
Lifecycle events & Admitted \textsf{live}, \textsf{fired} (ready/breached), \textsf{proposed}, \textsf{admitted}, \textsf{rejected}, \textsf{expired}, and \textsf{discharged} transitions. \\
\bottomrule
\end{tabular}
\end{table}

\begin{table}[t]
\centering
\caption{Reusable ACR profiles in the portable specification. This normalization
layer is not implemented or benchmarked in the tagged artifact.}
\label{tab:eng-acr-profiles}
\small
\begin{tabular}{p{0.20\linewidth}p{0.68\linewidth}}
\toprule
\rowhead
Profile & Additional required policy and bindings \\
\midrule
Deadline & Trusted clock, deadline or validity window, timeout semantics, late-event policy. \\
Assumption & Proposition, supporting evidence, freshness interval, renewal and invalidation policy. \\
External effect & Outbox-intent reference, idempotency and correlation keys, acknowledgement deadline, retry and compensation policy. \\
Semantic contradiction & Claim and vocabulary versions, conflicting evidence, reviewer and validation policy, permitted repair scope. \\
\bottomrule
\end{tabular}
\end{table}

\paragraph{Inbox and outbox boundaries.}
External events enter through an inbox deduplication boundary; external side effects
leave through an outbox boundary that stages provider calls with idempotency keys.
The outbox is an intent log: provider execution remains outside committed truth until
its result is observed and admitted as a subsequent transition. This prevents a
domain commit from silently becoming an external effect, and prevents a duplicate
external observation from creating a duplicate commit. Both boundaries sit on the
commit side of the authority line but neither owns truth: the inbox only deduplicates
inbound observations, and the outbox only stages outbound intent.

\textbf{Evidence status.} The reported experiments validate committed-intent and
deduplication mechanisms, not the complete physical-effect or temporal-exception
routes. They do not exercise dispatch failure after committed intent, timeout without
premature failure inference, late acknowledgement, or provider-level duplicate-effect
suppression. Those paths remain architectural specifications and explicit targets for
the next deterministic fault-injection suite.

\section{Admission, Commit, Logging, and Replay}
\label{app:eng-admission}

This appendix gives the deterministic admission contract under $\mathcal{C}$ and the
commit, logging, and replay sequence at the mechanism level. The contract is
independent of proposer identity: an LLM, solver, runtime driver, benchmark adapter,
or \ACR{} watcher or continuation may change the proposal package, but only the ordered admission
steps can create committed truth.

\paragraph{The admission contract.}
Admission is the deterministic pre-commit procedure that decides whether a proposal
may become a committed transition. It reads current effective state from
\StateView{} and retained evidence, never speculative or rejected proposal history,
and applies the ordered checks below. Steps 1--8 are the executable contract used by
the implementation and the RQ1--RQ6 experiments.

\begin{enumerate}
\item Parse the proposal package and reject malformed packages before commit
  processing begins.
\item Check tenant, entity, idempotency key, operation key, and declared conflict
  scope.
\item Read current \StateView{} and retained evidence; do not validate against
  speculative or rejected proposal history.
\item Apply the deterministic constraints in $\mathcal{C}$: finite-state rules,
  policy rules, stale-world checks, dependency closure, and compensation safety.
\item If the proposal is a repair, require that the underlying failure is resolved
  under $\mathcal{C}$, or that the trigger evidence remains effective and queryable
  after the repair.
\item If a watcher proposes \ACR{} readiness or breach, require admitted trigger
  evidence and commit only the lifecycle transition. Permit continuation resumption
  only from that committed state; require its remedy package to re-enter admission and
  forbid direct mutation of the \CTL{} or \StateView{} at either crossing.
\item Serialize or reject overlapping incompatible conflict scopes; admit only if the
  resulting transition is valid over current effective state.
\item Atomically append the admitted transition and recovery metadata to durable
  storage and update \StateView{}; otherwise record a queryable rejection reason.
\end{enumerate}

A solver certificate or LLM rationale is treated as evidence supplied to these
checks, not as authority. The validator reads log-grounded effective state only: a
finite-state edge is checked for legality with \texttt{action\_type} disambiguating
self-loops; each dependency is required to be committed \emph{and} still effective
(existence alone is insufficient); and application constraints are keyed by
\texttt{(fsm, action\_type)} so that, for example, a pickup may be required to follow
the driver's arrival as read from the driver's effective state. The gate returns an
admitted transition, a rejected proposal with structured reasons, or an escalation
request for a wider scope.

\paragraph{Commit, logging, and projection.}
An admitted proposal commits through a single atomic step that also produces the
projection update. The commit batch is the unit of multi-entity atomicity. In one
transaction it: locks the affected entities in sorted key order, verifies the
expected monotonic \texttt{version} for each \texttt{(eid, fsm)}, re-checks that every
dependency in $D$ is committed and effective, appends all \CTL{} rows, updates the
synchronous \StateView{} projection for every affected key, and stages any outbox
intents. Three guarantees hold exactly at the write: idempotency (the \texttt{rid}
primary key, with \texttt{op\_id} as the client key when \texttt{rid} is
server-assigned), ordering (the monotonic per-entity \texttt{version}), and
dependency effectiveness ($D$ committed and not since compensated or superseded).
Only after this step does a row exist; a rejected proposal yields a queryable
rejection, not a row.

\paragraph{Idempotency and retry safety.}
Because durable runtime drivers execute steps at least once, the commit path must be
idempotent under retry. The \texttt{rid} primary key (or \texttt{op\_id} when
\texttt{rid} is server-assigned) makes a retried commit return the existing committed
record rather than duplicate domain truth. This is why the idempotency key is part of
the proposal package and is checked at Step~2 before any state mutation: a retried
submission after a crash or timeout resolves to the same committed transition, so the
commit batch is safe to re-drive.

\paragraph{Replay.}
The \CTL{} is the source of committed truth, and current state is recovered by replay
rather than read from a stored cursor. Replay of an \texttt{(eid, fsm)} history
proceeds in order and self-checks consistency: each record's \texttt{state\_before}
must equal the prior record's \texttt{state\_after}, and a mismatch is a replay error
that signals log corruption. The bare replay yields the finite-state value; the full
reconstruction is the \StateView{} that additionally folds the $X$ attribute deltas
forward and retains the effective records supporting the state. Because compensation
and supersession are themselves later admitted transitions, effective-state replay
reproduces current truth without deleting history or moving a cursor: the effective
index marks which records still count, and the projection excludes the rest.

\section{Disruption, Recovery, and Compensation}
\label{app:eng-recovery}

This appendix specifies recovery at the mechanism level. The central property is that
recovery does not bypass admission: a disruption, an admitted ready or breached contract, or a
crash-recovery action produces a proposal that passes through the same gate as any
generated action. Recovery decisions are made in the proposal and policy layers; the
substrate supplies only the mechanism and the invariants.

\paragraph{Contract readiness and remedies do not bypass admission.}
When an \ACR{} trigger is detected, the watcher constructs a readiness or breach
lifecycle proposal from admitted evidence. Step~6 admits or rejects that proposal and
commits only the lifecycle record; rejection leaves the prior live state unchanged.
Only an \ACR{} in the resulting committed ready or breached state may resume its
continuation. The continuation constructs a remedy package declaring the \ACR{}'s
bounded dependency \texttt{scope}; that package re-enters ordinary admission and may
touch only the recorded scope unless a wider scope is explicitly admitted as an
escalation. Thus neither detection nor resumption can install domain truth on its own.
Rejected or compensated \ACR{} proposals with matching failure signatures remain
queryable as historical evidence for future ranking or guard strengthening, which is
how memory becomes active without becoming authoritative.

\paragraph{Three recovery moves.}
The class of an action against the full governed-invariant set, re-checked at commit
time, selects the recovery move; the substrate supplies the mechanism and the policy
layer supplies the timing. Table~\ref{tab:eng-recovery-moves} states the three moves.
A restore (return to a clean checkpoint) is the bulk form of an undo: it resets state
but, unlike an exact undo, does not reverse external effects, so it is paired with
compensation.

\begin{table}[t]
\centering
\caption{Recovery moves. The action's class against the current governed-invariant
set selects the move; the substrate supplies the mechanism.}
\label{tab:eng-recovery-moves}
\small
\begin{tabular}{p{0.16\linewidth}p{0.34\linewidth}p{0.38\linewidth}}
\toprule
\rowhead
Move & When & Mechanism \\
\midrule
\texttt{undo} & An exact inverse exists (reversible). & Apply the inverse; the entity returns to $s^{-}$. \\
\texttt{compensate} & No inverse, but governed invariants are restorable within a validity window (compensable). & Append a compensating transition; history is preserved. \\
escalate & Neither holds, or local repair cannot contract (irreversible). & Veto, or hand to a human task; no committed transition is written. \\
\bottomrule
\end{tabular}
\end{table}

\paragraph{Dependency-closed compensation and supersession.}
Compensation is represented as a new admitted transition, never as physical deletion,
and the whole saga group identified by \texttt{tx\_group\_id} unwinds together. The
group is compensated in reverse topological order over the group's dependency
subgraph, not merely reverse insertion order, so a leg is never compensated before
the legs that depend on it. Admission enforces dependency-closed compensation through
the recovery invariants: a compensation is rejected if it would leave an effective
record depending on an ineffective record (no orphaned effective dependents), and a
compensation is rejected if it would break the effective chain for an entity
(effective-chain preservation). Supersession is the related operation by which a
later admitted transition replaces an earlier one through the \texttt{supersedes}
link in $M$; like compensation, it changes effectiveness without erasing history.

\paragraph{Effective-chain preservation under cross-entity correction.}
Compensation and supersession never delete history; they only change which records
are effective. Because the projection is materialized per \texttt{(eid, fsm)}, a
correction that touches several entities must drop the corrected record from current
truth on \emph{every} entity it touches, not only the entity that issued the
correction. Refreshing only the issuing entity would let a now-ineffective record
linger as current truth elsewhere. Replay from the \CTL{} therefore remains the
source of truth, and the per-entity projection is rebuilt for all affected keys after
a correction.

\paragraph{Evidence-preserving repair.}
A memory-initiated repair must not discharge its own trigger by deleting,
overwriting, compensating, or obscuring the evidence that justified it. Admission
permits a repair to alter trigger evidence only if it verifies, under the retained
evidence and validator rules, that the proposed repair addresses the triggering
condition rather than merely suppressing its observable signal (Step~5). Concretely:
let a repair proposal be triggered by evidence $E$ of a failure condition $F$.
Admission admits the repair only if either $F$ is resolved under $\mathcal{C}$, or $E$
remains effective and queryable after the repair. As a result, no admitted repair can
clear its own trigger by destroying the evidence that caused it, and a repeated
failure remains visible through its failure signature for later analysis.

\paragraph{Crash recovery as invariants in reverse.}
A transactional fault, such as a crash mid-leg, is handled with the same mechanisms.
Recovery locates the most recent record marked as a clean restore point, appends
compensating transitions for the saga groups committed after it in reverse
topological order (marking them \texttt{compensated}), and resumes the reactive loop
from the resulting effective state. Nothing is deleted and no cursor moves: the
original records remain for audit, and effective-state replay yields current truth
because each compensation is itself a later committed transition. Recovery is thus the
ordinary commit invariants (idempotency, monotonic versions, dependency
effectiveness) run in reverse under a policy.

\section{Staged Implementation (R0--R8) and Reproducibility}
\label{app:eng-stages}

This appendix records the staged build of \Mnemo{} and the reproducibility path. The
build follows a small-to-large discipline: a visible, auditable correctness kernel is
established before scale-out infrastructure is connected, so that the authority
boundary (the central claim) is not obscured by an external engine or solver.
Table~\ref{tab:eng-stages} lists stages R0 through R8 and their status. The deployment
service boundary (R8) is now complete: a deployable service with health,
proposal-submission, \StateView{}, and metrics endpoints is implemented and audited
for authority separation under concurrent HTTP submission (Appendix~\ref{app:r8}).

The R0--R8 rows below are capability groups used for this supplement, not a claim
that each row names one repository release tag. The tagged artifact's finer-grained
R4.x--R8 history is recorded in its release notes.

\begin{table*}[t]
\centering
\caption{Staged implementation R0--R8. Each stage is independently runnable and
testable. The correctness kernel and its safety boundary are completed before
deployment scale-out; R8 changes deployment scale, not the semantics proved in the
main text.}
\label{tab:eng-stages}
\small
\begin{tabularx}{\linewidth}{p{0.055\linewidth}p{0.235\linewidth}X p{0.105\linewidth}}
\toprule
\rowhead
Stage & Delivers & Detail & Status \\
\midrule
R0 & Substrate kernel & \CTL{} schema, append-only commit, monotonic per-entity versioning, PostgreSQL store, \texttt{log\_position} ordering. & Complete \\
R1 & Effective-state projection & \StateView{} reconstruction from effective records; full and effective entity histories. & Complete \\
R2 & Deterministic admission under $\mathcal{C}$ & Isolated validator: finite-state legality, dependency effectiveness, conflict scope, stale-world checks; structured rejection reasons. & Complete \\
R3 & Compensation and recovery & Dependency-closed compensation, supersession, restore points, fail-closed compensation invariants; recovery-event log. & Complete \\
R4 & Idempotency and boundaries & \texttt{op\_id} logical idempotency, inbox deduplication, outbox intent boundary, multi-entity atomic commit batch. & Complete \\
R5 & Active contract records & \ACR{} creation and readiness as admitted lifecycle transitions, non-authoritative continuation output, scoped repair, failure signatures. & Complete \\
R6 & Benchmark and proposer integration & J1--J4 end-to-end harness; proposer-quality safety-invariance path; mechanism audit plus executable Temporal product check; recursive-recovery stress. & Complete \\
R7 & Storage substrate and cost audit & Optional environment-gated PostgreSQL and pooled-PostgreSQL paths; RQ5 worker and matched data-size audits. & Complete \\
R8 & Deployable service boundary & HTTP service with health, proposal-submission, \StateView{}, and metrics endpoints; authority separation audited under concurrent submission. & Complete \\
\bottomrule
\end{tabularx}
\end{table*}

\paragraph{Default reproducibility path.}
The default reproducibility path is PostgreSQL-free and deterministic:
\begin{quote}\small
Tag: \texttt{arxiv-atp-rq1-rq9b-r8-v2}\\
Commit: \texttt{7aab09c43c1672c01b415b4987cd2d1d9a6fd0a4}
\end{quote}
From the repository root, run:
\begin{verbatim}
python -m pytest -q tests/experiments/test_rq[1-8]_*.py
python -m pytest -q tests/benchmarks/test_production_tp_comparator.py
python -m pytest -q -m temporal \
  tests/benchmarks/test_temporal_product_baseline.py
python -m pytest -q \
  tests/benchmarks/test_recursive_recovery_stress.py
python -m pytest -q tests/benchmarks/test_rq5_data_scale.py
python -m benchmarks.realm.rq5_data_scale
\end{verbatim}
The five test commands yield \texttt{8}, \texttt{2}, \texttt{1}, \texttt{1}, and
\texttt{1} passing tests; the final command regenerates the RQ5 scaling reports. Together they
cover authority
separation, serial-equivalent admission, evidence-preserving repair, obligation
containment, effective-state and compensation safety, storage-substrate correctness,
the J1--J4 end-to-end execution, proposer-quality safety invariance, and the
workflow/saga mechanism comparison, executable Temporal check, recursive-recovery stress,
and matched data-size scaling. A broader safe command additionally runs
\path{tests/core}, \path{tests/apps}, and \path{tests/benchmarks}. Reviewers
should not run all of \texttt{tests} / \texttt{experiments} as a blanket command, because older
exploratory imports outside the \ATP{} artifact path are intentionally excluded from
the reproducibility suite.

\paragraph{Infrastructure-cost benchmark (RQ5).}
The cost audit is run as a separate infrastructure-cost benchmark over selected real
infrastructure tests rather than as a violation-suppression experiment. It reports
end-to-end, admission, and commit latency, process-level throughput at one, four, eight,
and 16 workers, and matched commit-only versus ATP throughput at 1,000, 5,000, and
10,000 operations. The worker audit saturates beyond eight workers; the data-size sweep
measures $5.2$--$6.6\%$ incremental ATP cost while exposing local storage/projection slowdown.
The scripts that generate the RQ5 tables are included alongside
\path{ARTIFACT_EVALUATION.md},
which lists the exact safety validation command and expected reports.

\paragraph{Optional live PostgreSQL path.}
The live PostgreSQL path for the storage-substrate experiment is optional and
environment-gated by the \texttt{MNEMOSYNE\_POSTGRES\_DATABASE\_URL} variable, with an
optional connection-pool dependency. When supplied, it validates recovery-event
append and list operations, canonical duplicate-idempotency behavior, clean
sequence-conflict handling, and pooled connection use. The artifact claim does not
require a live database: default CI remains PostgreSQL-free and pool-dependency-free,
so the default \texttt{8 passed} plus comparator \texttt{2 passed} results reproduce
without any database setup.

\section{Structure and Scalability of the Constraint Set $\mathcal{C}$}
\label{app:constraint-grammar}
A natural objection is that ATP merely moves the correctness burden from the agent prompt into the
validator: if every semantic rule must be hand-coded, $\mathcal{C}$ does not scale. In the
implementation most of $\mathcal{C}$ is not hand-coded per case. Its rules fall into three classes
(Table~\ref{tab:constraint-taxonomy}).

\begin{enumerate}[leftmargin=1.4em,itemsep=2pt,topsep=2pt]
\item \textbf{Structurally derived.} Generated mechanically from declarations the application already
maintains. Finite-state legality is derived from per-entity transition table keyed by
\texttt{action\_type}; dependency effectiveness and compensation closure are derived from the
effective-chain over \CTL{}; idempotency and conflict scope are derived from the record's identity
triad. No per-rule code is written.
\item \textbf{Declarative predicates.} Declared once as data, not procedures, and checked by a generic
evaluator: schema and type constraints, required-evidence handles, world-fact freshness keys, and
conflict-scope keys.
\item \textbf{Application validators.} The only genuinely hand-written part: domain semantics that
cannot be derived or declared (for example, a capacity rule specific to one operation). In the
artifact these are a small minority of the constraint surface.
\end{enumerate}

\begin{table}[t]
\centering
\caption{Taxonomy of the constraint set $\mathcal{C}$: most rules are derived or declared, not
hand-coded per case.}
\label{tab:constraint-taxonomy}
\setlength{\tabcolsep}{4pt}\footnotesize
\begin{tabularx}{\linewidth}{>{\raggedright\arraybackslash}X >{\raggedright\arraybackslash}X >{\raggedright\arraybackslash}X}
\toprule
\rowhead
Class & How declared & Examples \\
\midrule
Structurally derived & generated from FSM table, effective-chain, identity triad & finite-state legality, dependency effectiveness, compensation closure, idempotency, conflict scope \\
Declarative predicates & declared as data, evaluated generically & schema and type, required evidence, world-fact freshness \\
Application validators & hand-written per operation & domain capacity and business rules \\
\bottomrule
\end{tabularx}
\end{table}

A proposal package is admissible when it satisfies the conjunction of all applicable constraints,
which in grammar form is:
\begin{verbatim}
admit(p)       := structural(p) AND declarative(p) AND validators(p)
structural(p)  := fsm_legal(p) AND deps_effective(p)
                  AND compensation_closed(p) AND idempotent(p)
                  AND scope_consistent(p)
declarative(p) := schema_ok(p) AND evidence_present(p)
                  AND world_fresh(p)
validators(p)  := conjunction of application-declared predicates
\end{verbatim}
Because the structural and declarative classes are generated from declarations the application already
keeps (entity FSMs, the dependency model, the schema), adding a new operation typically extends
$\mathcal{C}$ by declaration rather than by new validator code. That is what keeps the admission
boundary scalable as an application grows, and it bounds the hand-written surface to genuine domain
semantics.

\section{Workflow-Engine and Storage-Substrate Selection}
\label{app:eng-engine}

This appendix records the substrate selection and its rationale. The governing
principle is that the substrate is non-authoritative: a workflow engine or storage
backend may change deployment scale, availability, or throughput, but it must
preserve the \ATP{} authority boundary. Committed truth is produced only by admitted
transitions, regardless of which engine orchestrates the work or which store persists
the log.

\paragraph{Why the substrate is non-authoritative.}
\Mnemo{} separates orchestration from authority. Runtime drivers schedule work, fire
timers, retry steps, and wake obligations, but they do not own committed truth: a
driver may detect a trigger and submit a readiness or breach proposal grounded in
admitted evidence. Only after admission commits that lifecycle state may the driver
call a continuation provider and submit its remedy package; it may not directly
mutate lifecycle or domain truth. Keeping the source of truth in the \CTL{} rather
than in engine-internal state
is a deliberate design fork. Letting the engine own state would tie the design to one
engine and weaken the validator's evidence story, because admission reads the
effective-state projection rather than engine-private execution history. Holding the
authority boundary in the admission gate is what lets a deployment substrate change
without changing the correctness semantics.

\paragraph{Recommended two-tier substrate.}
The deciding requirement is a true ACID store with snapshot isolation, which
PostgreSQL provides natively through MVCC. The recommendation is therefore a
PostgreSQL-anchored two-tier substrate: a Postgres-backed durable-execution and
transaction core that hosts the \CTL{}, the effective-state projection, and the
recovery-event log in one store, with a compute fan-out tier layered beneath it for
concurrent proposal and probe execution. The durable tier supplies transactional
bookkeeping and a measurable commit-failure rate; the compute tier supplies node-count
scale. In the reference implementation the store protocol runs on PostgreSQL (and pooled
PostgreSQL for concurrency) as the substrate for all reported experiments, with SQLite
retained as a zero-setup fallback for local development and portability
(Figure~\ref{fig:storage}); the backing store affects neither correctness semantics nor any
reported number. The scalability and Kubernetes deployment phase is a
deployment concern that may add containerized services, durable-workflow workers,
orchestration, observability, and load tests; it changes deployment scale, not
\ATP{} correctness semantics.

\paragraph{Durable engines are interop targets.}
The 2025--2026 durable-execution category provides persistence, retries, and
exactly-once or ACID semantics as infrastructure, and \Mnemo{} builds the \CTL{} and
the \ATP{} layer on top of such an engine rather than reimplementing logging or
recovery primitives. Table~\ref{tab:eng-engine-survey} summarizes the engines
considered. They are treated as interoperability or deployment targets behind an
engine-agnostic interface, not as the authority for committed truth: durable
workflow engines such as Temporal and Cadence orchestrate execution but are
non-authoritative, and Argo Workflows is retained as a Kubernetes deployment and
export target rather than a transaction engine. Checkpoint-persistence agent
frameworks are explicitly excluded as the transactional substrate, because checkpoint
persistence is not durable execution with exactly-once or ACID guarantees; they may
serve only as optional agent-authoring layers above a real durable engine.

\begin{table}[t]
\centering
\caption{Durable-execution and orchestration engines considered. Each is an
interoperability or deployment target behind an engine-agnostic interface; none owns
the \ATP{} authority boundary, which remains in the admission gate.}
\label{tab:eng-engine-survey}
\small
\begin{tabular}{p{0.20\linewidth}p{0.30\linewidth}p{0.38\linewidth}}
\toprule
\rowhead
Engine & Model & Role for \Mnemo{} \\
\midrule
PostgreSQL & ACID store with MVCC snapshot isolation & \makecell[tl]{Anchor substrate: hosts \CTL{},\\projection, recovery log.} \\
SQLite & Embedded ACID store & Zero-setup fallback for local development and portability. \\
Temporal & Durable execution, code-defined sagas & Orchestration and scale escape hatch; non-authoritative. \\
Cadence & Durable execution (Temporal predecessor) & Orchestration interop target; non-authoritative. \\
Argo Workflows & Kubernetes DAG execution & Deployment and export target, not a transaction engine. \\
\bottomrule
\end{tabular}
\end{table}

\paragraph{Engine portability.}
The durable-execution core sits behind an engine-agnostic interface so that the
choice of orchestrator is a contained substitution. The Temporal-style boundary makes
the contract concrete: the workflow layer is deterministic and side-effect free
(timers, timeouts, retries, signals), while the activity layer performs the side
effects (validate the commit batch, build the \CTL{} records, commit, reproject
\StateView{}, stage outbox). The workflow orchestrates and the activity commits, but
the store and \CTL{} own truth. Because activities execute at least once under a
retry policy, the \texttt{op\_id} idempotency key makes a retried activity return the
existing committed record rather than duplicate domain truth. This portability is why
the same authority boundary holds whether the orchestrator is a local driver or a
durable workflow engine, and why a change of engine does not change the safety
properties proved in the main text.

\section{Worked Proposal Examples}
\label{app:proposal-examples}
This appendix makes the \emph{proposal is not truth} boundary concrete. The examples use compact
JSON-like records to show what crosses the proposal boundary and what the admission gate does before
anything becomes committed truth. They are illustrative but match the implementation contract used
throughout the artifact: a proposal package carries a tenant, entity, operation, dependency set, world
assumptions, conflict scope, evidence handles, and an idempotency key. Admission evaluates the package
against the current effective \StateView{}, retained evidence, and the declared constraint set
$\mathcal{C}$. Two axes matter independently: whether the proposal is \emph{well formed}, and whether the
committed \emph{output} it would produce is $\mathcal{C}$-valid.

\subsection{Valid Proposal}
A valid proposal may come from an LLM, solver, workflow driver, or active contract continuation. The proposer
supplies a candidate action, but the candidate stays non-authoritative until admission accepts it.
\begin{verbatim}
{
  "proposal_id": "p_valid_hotel_001",
  "tenant": "t1",
  "entity": "trip_17",
  "operation": "book_hotel",
  "depends_on": ["flight_f1"],
  "world_assumptions": {"flight_f1_status": "confirmed"},
  "conflict_scope": "t1/trip_17/lodging",
  "evidence": ["user_request_e1"],
  "idempotency_key": "book_hotel_trip17_once"
}
\end{verbatim}

\noindent\textbf{Admission result.} Accepted. The referenced flight remains effective in \StateView{}, the
world assumption is current, no incompatible proposal occupies the lodging conflict scope, and the
idempotency key is fresh.

\noindent\textbf{Committed transition.} The admitted hotel booking is appended to the
committed-transition log (\CTL{}); \StateView{} projects the booking as effective current truth; a
confirmation message may be staged in the outbox under its own idempotency key.

\subsection{Invalid Proposal}
The next proposal is syntactically well formed but semantically stale. A classical parser would accept the
record; ATP rejects it before commit because its assumptions disagree with observed effective state.
\begin{verbatim}
{
  "proposal_id": "p_stale_hotel_002",
  "tenant": "t1",
  "entity": "trip_17",
  "operation": "book_hotel",
  "depends_on": ["flight_f2"],
  "world_assumptions": {"flight_f2_status": "confirmed"},
  "observed_world": {"flight_f2_status": "cancelled"},
  "conflict_scope": "t1/trip_17/lodging",
  "evidence": ["user_request_e1"],
  "idempotency_key": "book_hotel_trip17_once"
}
\end{verbatim}

\noindent\textbf{Admission result.} Rejected. The proposal depends on a flight that is no longer effective
under the admitted observations. The rejection is recorded with a durable reason, but no domain transition
is appended.

\noindent\textbf{Committed effect.} \CTL{} is unchanged with respect to domain truth, \StateView{} is
unchanged, and no external side effect is staged.

\subsection{Valid Output After a Valid Proposal}
A valid admitted proposal may create a valid outbox intent. The outbox record is not the external effect
itself; it is an admitted intent to perform the effect under an idempotency key.
\begin{verbatim}
{
  "proposal_id": "p_notify_after_booking_003",
  "tenant": "t1",
  "entity": "trip_17",
  "operation": "stage_outbox_notification",
  "depends_on": ["hotel_booking_b1"],
  "message_type": "booking_confirmation",
  "recipient": "user_u1",
  "conflict_scope": "t1/trip_17/notification",
  "idempotency_key": "notify_trip17_booking_once"
}
\end{verbatim}

\noindent\textbf{Admission result.} Accepted. The booking dependency is an effective committed fact, the
outbox idempotency key has not been consumed, and no incompatible notification intent exists for the same
scope.

\noindent\textbf{Committed effect.} \CTL{} records the admitted outbox intent; \StateView{} exposes the
notification as staged; the provider call stays outside committed truth until its result is observed and
admitted as a later transition.

\subsection{Invalid Output After a Valid Proposal}
Even after a valid domain proposal, a later output proposal can be invalid. Here an agent tries to charge a
card twice by changing only the provider-level idempotency key while targeting the same committed payment
obligation.
\begin{verbatim}
{
  "proposal_id": "p_duplicate_charge_004",
  "tenant": "t1",
  "entity": "trip_17",
  "operation": "stage_outbox_payment",
  "depends_on": ["hotel_booking_b1"],
  "payment_obligation": "hotel_booking_b1/deposit",
  "amount": "250.00",
  "conflict_scope": "t1/trip_17/payment/deposit",
  "idempotency_key": "pay_deposit_trip17_b1_once",
  "provider_idempotency_key": "new_key_but_same_obligation"
}
\end{verbatim}

\noindent\textbf{Admission result.} Rejected (or escalated for reconciliation). The proposal targets an
already staged or satisfied payment obligation under the same semantic conflict scope and obligation
idempotency key. A changed \emph{provider} key is evidence, not authority to create a duplicate external
effect.

\noindent\textbf{Committed effect.} No duplicate outbox payment intent becomes current truth, and the
rejected package remains queryable as evidence for audit or guard strengthening. Only the first and third
proposals commit; the second and fourth are rejected with a durable reason, so committed-state correctness
is independent of how the proposal was generated.

\section{An End-to-End Transaction Trace}
\label{app:transaction-example}
This appendix traces one disruption-recovery path through ATP, separating five objects that agentic
systems often conflate: input observation, proposal package, admission decision, committed log record, and
external side-effect intent. It continues the trip scenario of Appendix~\ref{app:proposal-examples}.

\subsection{Input Observation}
\begin{verbatim}
flight_cancelled(flight_f1)
source_event_id = flight_provider_event_e123
observed_at     = 2026-06-30T10:15:00Z
\end{verbatim}
The input observation is evidence. It is not, by itself, a domain repair, and it grants no write authority
to the proposer or the runtime driver.

\subsection{Proposal Package}
A planner or active contract continuation proposes a repair package:
\begin{verbatim}
{
  "proposal_id": "p_repair_trip_017",
  "tenant": "t1",
  "entity": "trip_17",
  "operation": "repair_trip_after_flight_cancellation",
  "depends_on": ["flight_f1", "hotel_h1"],
  "evidence": ["flight_provider_event_e123"],
  "conflict_scope": "t1/trip_17",
  "proposed_steps": [
    "compensate_hotel(hotel_h1)",
    "search_alternate_flight(trip_17)",
    "notify_user(user_u1)"
  ],
  "idempotency_key": "repair_trip17_after_flight_f1_once"
}
\end{verbatim}
The proposal is not committed truth; it is a candidate package supplied to the admission gate.

\subsection{Admission}
Admission evaluates the package against $\mathcal{C}$ and the current effective state:
\begin{verbatim}
checks:
  dependency closure:    hotel_h1 depends on flight_f1
  stale-world facts:     flight_f1 is observed cancelled
  compensation safety:   no effective downstream record is orphaned
  evidence preservation: flight_provider_event_e123 stays queryable
  conflict scope:        no active repair owns t1/trip_17
  idempotency:           repair_trip17_after_flight_f1_once is fresh
\end{verbatim}
If all checks pass, the proposal is admitted; if any fails, the rejection and its reason remain durable,
but the proposed repair does not become domain truth.

\subsection{Committed-Transition Log}
An admitted repair appends records such as:
\begin{verbatim}
CTL append:
  compensation_requested(hotel_h1)
  repair_proposal_admitted(trip_17)
  ACR_updated(trip_repair_obligation)
  recovery_metadata(...)
\end{verbatim}
\CTL{} is the source of committed truth. Compensation is logical: the earlier hotel record is not erased;
it becomes ineffective only through admitted compensating transitions and effective-state projection.

\subsection{StateView}
\StateView{} projects current effective truth from \CTL{}:
\begin{verbatim}
StateView:
  flight_f1              = cancelled
  hotel_h1               = compensation_pending
  trip_17                = repair_in_progress
  trip_repair_obligation = live
\end{verbatim}
\StateView{} is not raw history: it excludes compensated, superseded, or otherwise ineffective records
while preserving full \CTL{} history for audit.

\subsection{Outbox}
External actions are staged rather than executed as hidden side effects of the planner:
\begin{verbatim}
outbox:
  notify_user(user_u1)
  idempotency_key = notify_trip17_repair_once
  status          = staged
\end{verbatim}
An outbox intent is not the external effect itself. The provider call and its result stay outside committed
truth until observed and admitted as later evidence, which prevents a generated proposal from silently
turning into an irreversible external mutation.

\subsection{Boundary Summary}
\begin{quote}
Input observation is evidence, not committed repair. A proposal is candidate intent, not truth. An outbox
intent is a staged effect, not external truth. Only admitted \CTL{} records define committed truth, and
\StateView{} is the effective projection used for future admission.
\end{quote}

\bibliographystyle{plainnat}
\bibliography{reference,mnemosyne_workflow_docs}

\end{document}